\documentclass[twoside]{article}

\usepackage[accepted]{aistats2023}

\usepackage[round]{natbib}
\renewcommand{\bibname}{References}
\renewcommand{\bibsection}{\subsubsection*{\bibname}}

\usepackage{amsmath,amsfonts,bm}

\def\1{\bm{1}}

\DeclareMathAlphabet{\mathsfit}{\encodingdefault}{\sfdefault}{m}{sl}
\SetMathAlphabet{\mathsfit}{bold}{\encodingdefault}{\sfdefault}{bx}{n}

\newcommand{\R}{\mathbb{R}}

\def\<{{\langle}}
\def\>{{\rangle}}

\usepackage[utf8]{inputenc} 
\usepackage[T1]{fontenc}    
\usepackage[pagebackref]{hyperref} 
\usepackage{url} 
\usepackage{booktabs}
\usepackage{amsfonts}
\usepackage{nicefrac}
\usepackage{microtype}
\usepackage[dvipsnames]{xcolor}
\usepackage{amssymb}
\usepackage{amsmath}
\usepackage{amsthm}
\usepackage{subcaption}
\usepackage{float}
\usepackage{wrapfig}
\usepackage{lipsum}
\usepackage{graphicx}
\usepackage{mathtools}

\usepackage{multirow}

\newtheorem{assumption}{Assumption}
\newtheorem{definition}{Definition}
\newtheorem{remark}{Remark}
\newtheorem{theorem}{Theorem}
\newtheorem{lemma}{Lemma}

\usepackage{algorithm}
\usepackage{algpseudocode}

\hypersetup{
  colorlinks=true,
  citecolor=Blue,
  linkcolor=Blue,
  urlcolor=Blue,
}%

\newtheorem{innercustomthm}{Theorem}
\newenvironment{customthm}[1]
  {\renewcommand\theinnercustomthm{#1}\innercustomthm}
  {\endinnercustomthm}

\def\R{{\mathbb{R}}}
\def\b0{{\boldsymbol{0}}}

\usepackage{accents}
\newcommand{\ubar}[1]{\underaccent{\bar}{#1}}

\usepackage{xr}
\makeatletter
\newcommand*{\addFileDependency}[1]{
  \typeout{(#1)}
  \@addtofilelist{#1}
  \IfFileExists{#1}{}{\typeout{No file #1.}}
}
\makeatother

\newcommand*{\myexternaldocument}[1]{
    \externaldocument{#1}
    \addFileDependency{#1.tex}
    \addFileDependency{#1.aux}
}

\renewcommand*{\backref}[1]{}  
\renewcommand*{\backrefalt}[4]{
  \ifcase #1 
     No cited.
  \or
     (Cited on page #2.)
  \else
     (Cited on pages #2.)
  \fi}

\myexternaldocument{aistats2023_supplement}

\begin{document}

\twocolumn[

  \aistatstitle{Neural Monge Map estimation and its applications
  }

  \aistatsauthor{Jiaojiao Fan$^*$,~Shu Liu$^*$,~Shaojun Ma,~Haomin Zhou,~Yongxin Chen
  }

  \aistatsaddress{ Georgia Institute of Technology } ]

\begin{abstract}
  Monge map refers to the optimal transport map between two probability distributions and provides a principled approach to transform one distribution to another.  Neural network based optimal transport map solver has gained great attention in recent years. Along this line, we present a scalable algorithm for computing the neural Monge map between two probability distributions. Our algorithm is based on a weak form of the optimal transport problem, thus it only requires samples from the marginals instead of their analytic expressions, and can accommodate optimal transport between two distributions with different dimensions. Our algorithm is suitable for general cost functions, compared with other existing methods for estimating Monge maps using samples, which are usually for quadratic costs. The performance of our algorithms is demonstrated through a series of experiments with both synthetic and realistic data, including text-to-image generation and image inpainting tasks.
\end{abstract}

\section{Introduction}
The past decade has witnessed great success of optimal transport (OT) \citep{villani2008optimal} based applications in machine learning community \citep{wgan,otrobotics1, otmatch, icnnot, otrobotics2, learnsde1, icnnwasserbary,otpgm,alvarez2020geometric,AlvSchMro21,bunne2021jkonet,MokKorLiBur21,bunne2022supervised,yang2018scalable,fan2022complexity}.
The Wasserstein distance induced by OT is widely used to evaluate the discrepancy between distributions thanks to its weak continuity
and robustness. In this work, given any two probability distributions $\rho_a$ and $\rho_b$ defined on $\mathbb{R}^n$ and $\mathbb{R}^m$, we consider the \textbf{Monge problem}
\begin{equation}
  C_{\textrm{Monge}}(\rho_a,\rho_b) \triangleq  \min_{\substack{T:\mathbb{R}^n\rightarrow\mathbb{R}^m,\\T_\sharp \rho_a = \rho_b}} \int_{\mathbb{R}^n} c(x,T(x))\rho_a(x)~dx.  \label{Monge problem}
\end{equation}
Here $c(x,y)$ denotes the cost of transporting from $x$ to $y$ and $T$ is the transport map. We define the pushforward of distribution $\rho_a$ by $T$ as $T_\sharp \rho_a(E) = \rho_a(T^{-1}(E))$ for any measurable set $E \subset \mathbb{R}^m$. The Monge problem seeks the cost-minimizing transport plan $T_*$ from $\rho_a$ to $\rho_b$. The optimal $T_*$ is also known as the \textbf{Monge map} of \eqref{Monge problem}.

Solving \eqref{Monge problem} in high dimensional space yields a challenging problem due to the curse of dimensionality for discretization. A modern formulation of \hyperref[Monge problem]{Monge problem} as a linear programming problem known as Kantorovich problem \citep{villani2003topics}, and adding an entropic regularization, one is capable of computing the problem via iterative Sinkhorn algorithm \citep{cuturi2013sinkhorn}. Such type of treatment has been widely accepted since it is friendly to high dimensional cases \citep{lineartimeotsinkhorn, genwithsinkhorn, otmatch, ipot}, but the algorithm does not scale well to a large number of samples and is not suitable to handle continuous probability measures.
Moreover, the transport map/plan obtained with this strategy cannot be generalized to unseen samples.

In this work, we propose a
scalable algorithm for estimating the Wasserstein distance as well as the optimal map in continuous spaces without introducing any regularization terms. Particularly, we apply the Lagrangian multiplier directly to Monge problem, and obtain a minimax problem.
\textbf{Our contribution} is summarized as follows:
1) We develop a neural network based algorithm to compute the optimal transport map associated with general transport costs between any two distributions given their \textit{unpaired} samples;
2) Our method is capable of computing OT problems between distributions over spaces that do not share the same dimension.
3) We provide a rigorous error analysis of the algorithm based on duality gaps; 4) We demonstrate its performance and its scaling properties in truly high dimensional setting through various experiments.

As other computational OT methods, our method does not require paired data for learning the transport map.
In real world,
the acquirement and maintenance of large-scale paired data are laborious. We will show that our proposed algorithm can be applied to multiple cutting-edge tasks, where the dominant methods still require paired data. In particular, our examples include text-to-image generation and image inpainting, which would confirm that our algorithm can achieve
competitive results even without paired data.

\begin{figure*}[ht!]
  \includegraphics[width=1\linewidth]{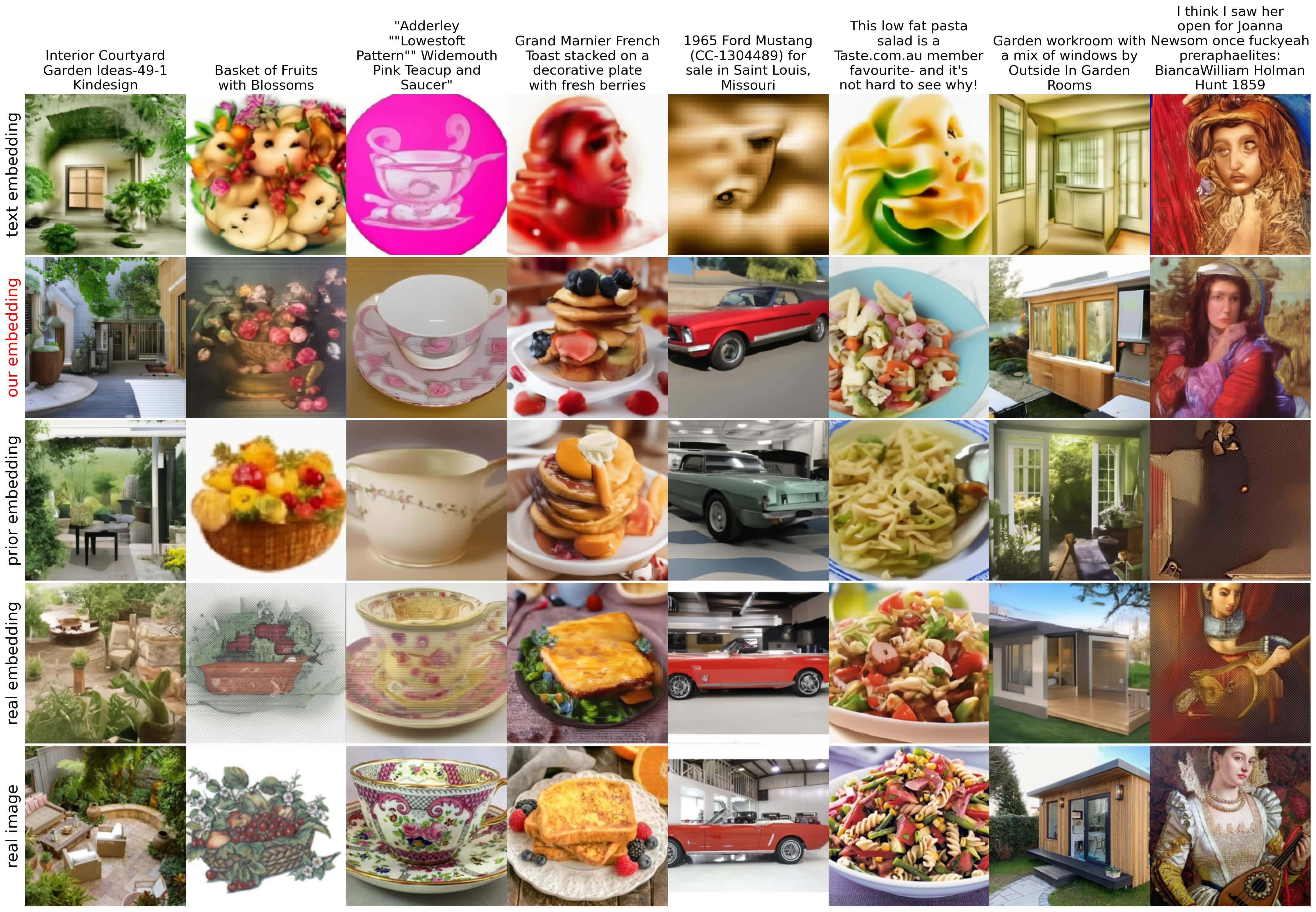}
  \caption{Random image samples with size $224\times224$ on Laion art testing prompts. Except for the last row of real images, we show the outputs of the same decoder with different conditions. The condition of decoder includes a text encoding, which we keep unchanged across rows, and an image embedding.
    In the first row, we feed the decoder the text embedding, and the generated images are unrealistic because the image and the text embeddings are not interchangeable. In the second row, we pass the embedding pushforwarded by our transport map, which is trained on unpaired data. As a baseline method, we pass an image embedding generated by the diffusion prior in DALL$\cdot$E2-Laion, which is trained on paired data.
    To explicitly show the decoder's recovery ability, we pass the real image embedding in the fourth row.}
  \label{fig:laion_art_prior}
\end{figure*}

\begin{figure*}[ht!]
  \includegraphics[width=1\linewidth]{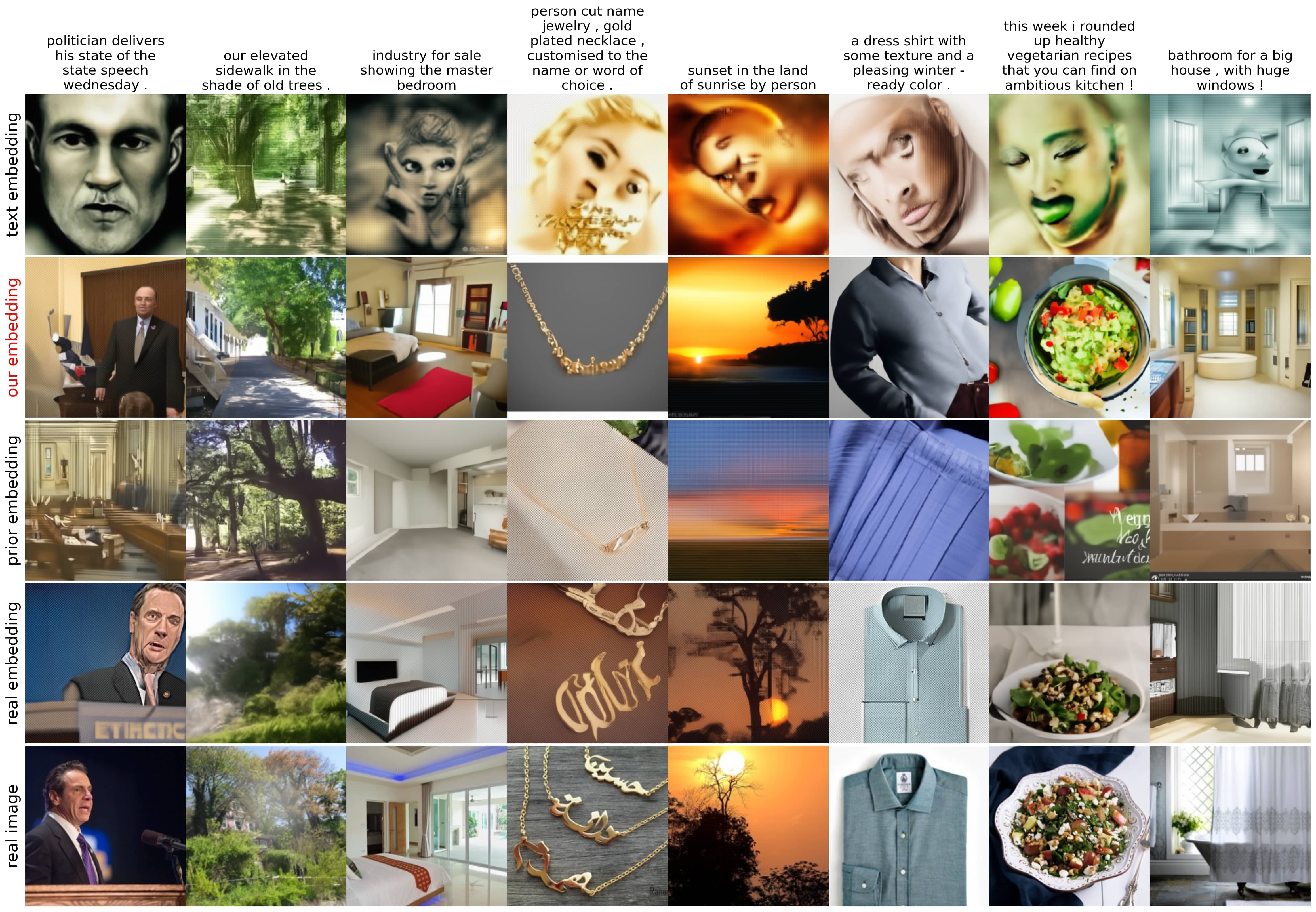}
  \caption{Random image samples on Conceptual Captions 3M prompts. The meaning of each row is the same as in Figure \ref{fig:laion_art_prior}.}
  \label{fig:cc3m_prior}
  \vspace{-0.3cm}
\end{figure*}

\section{A brief background of OT problems}
The general OT problem from $\mathbb{R}^n$ to $\mathbb{R}^m$ is  formulated as
\begin{equation}
  C(\rho_a, \rho_b) :=  \inf_{\pi \in \Pi(\rho_a,\rho_b) }
  \int_{\mathbb{R}^n\times \mathbb{R}^m}c(x,y)d\pi (x,y)
  , \label{ g OT}
\end{equation}
where we define $\Pi(\rho_a,\rho_b)$ as the set of joint distributions on $\mathbb{R}^{n \times m}$ with marginals equal to $\rho_a$ and $\rho_b$.
In this work, we will mainly focus on the cost function $c(\cdot,\cdot)$ that is well-defined and regular. The detailed assumptions of our considered cost functions are presented in
\eqref{condition c low bdd}, \eqref{condition c a} and \eqref{condition c b}
in
Appendix \ref{background}.

The Kantorovich dual formulation of the primal OT problem \eqref{ g OT} is (Chap 5, \cite{villani2008optimal})
\begin{equation}
  \sup_{\phi, \psi}\left\{\int_{\mathbb{R}^m} \phi(y)\rho_b(y)~dy - \int_{\mathbb{R}^n} \psi(x)\rho_a(x)~dx \right\}, \label{Kantorovich formula 1 }
\end{equation}
where the maximization is over all $\psi\in L^1(\rho_a)$, $\phi\in L^1(\rho_b)$ that satisfy $\phi(y)-\psi(x)\leq c(x,y)$ for any $x\in\mathbb{R}^n, y\in \mathbb{R}^m$.

For further discussions on the dual problem \eqref{Kantorovich formula 1 }, as well as the equivalence between the primal OT problem \eqref{ g OT} and its Kantorovich dual \eqref{Kantorovich formula 1 }, we refer the reader to Appendix \ref{background}.

In general, the optimal solution $\pi_*$
of \eqref{ g OT} can be treated as a random transport plan, i.e., we are allowed to break a single particle into pieces and then transport each piece to certain positions according to the plan $\pi_*$. However, in this study, we will mainly focus on computing for the deterministic optimal map $T_*$ of the classical version of the OT problem, which is the \hyperref[Monge problem]{Monge problem}.

The following theorem states the existence and uniqueness of the optimal solution to the Monge problem. It also reveals the relation between the optimal map $T_*$ and the optimal transport plan $\pi_*$. A more complete statement of this theorem can be found in Appendix \ref{background}.
\begin{theorem}%
  \label{thm monge}
  (Informal)
  Suppose the cost $c$ satisfies \eqref{condition c low bdd}, \eqref{condition c a},
  \eqref{condition c b}, and assume that $\rho_a$ and $\rho_b$ are compactly supported and $\rho_a$ is absolute continuous with respect to the Lebesgue measure on $\mathbb{R}^n$. Then there exists a unique-in-law\footnote{This means if $T_{**}$ is another Monge map solving the problem \eqref{Monge problem}, then $T_*=T_{**}$ on $\textrm{Spt}(\rho_a)\textbackslash  E_0$, where $\textrm{Spt}(\rho_a)$ is the support of $\rho_a$ and $E_0$ is a zero measure set.} transport map $T_*$ solving Monge problem \eqref{Monge problem}. Furthermore, $(\textrm{Id}, T_*)_\sharp \rho_a$ is an optimal solution to the OT problem \eqref{ g OT}.
\end{theorem}

\section{Proposed method}

In order to formulate a tractable algorithm for the general Monge problem \eqref{Monge problem}, we first notice that \eqref{Monge problem} is a constrained optimization problem. Thus, we introduce a Lagrange multiplier $f$ for the constraint $T_\sharp\rho_a = \rho_b$ and reformulate \eqref{Monge problem} as a saddle point problem
\begin{equation}
  \sup_f\inf_T \mathcal{L}(T,f)  \label{max - min} ,
\end{equation}
where $\mathcal{L} (T,f)$ equals
\begin{align}
    &
  \int_{\mathbb{R}^n}  c(x,T(x))\rho_adx+\int_{\mathbb{R}^m} f(y)(\rho_b-T_\sharp\rho_a)dy  \nonumber             \\  \label{eq:L def}
  = & \int_{\mathbb{R}^n} \!\!\! \left[ c(x,T(x))-f(T(x))\right]\rho_adx + \int_{\mathbb{R}^m} \!\! f(y)\rho_bdy.
\end{align}
The following theorem ensures the consistency of the max-min formulation \eqref{max - min}.
The proof
is provided in Section \ref{proof thm consistency}

\begin{theorem}[Consistency]\label{thm consistent}

  Suppose the max-min problem \eqref{max - min} admits at least one saddle point solution. We denote it as $(\bar{T},\bar{f})$. Under the assumption $\bar{T}_{\sharp} \rho_a=\rho_b$, we have
  \begin{itemize}
    \item
          $\phi=\bar{f}$ is an optimal solution
          to the Kantorovich dual problem \eqref{Kantorovich formula 1 }.
    \item The optimal solution to the Monge problem \eqref{Monge problem} exists, and $\bar{T}$ is an optimal solution.
    \item
          $\mathcal{L}(\bar{T}, \bar{f})=C_{\textrm{Monge}}(\rho_a, \rho_b)$.
  \end{itemize}
\end{theorem}

Without the assumption $\bar{T}_\sharp\rho_a=\rho_b$, we no longer have the guarantee on the existence of Monge map $T_*$, not to mention the consistency between $\bar{T}$ and $T_*$. However, we are still able to show the consistency between the saddle point value $\mathcal{L}(\bar{T}, \bar{f})$ and the general OT distance $C(\rho_a,\rho_b)$.
\begin{theorem}[Consistency without $\bar{T}_\sharp\rho_a=\rho_b$]\label{thm weak consist}
  Suppose the max-min problem \eqref{max - min} admits at least one saddle point solution  $(\bar{T},\bar{f})$, then $\mathcal{L}(\bar{T}, \bar{f})=C(\rho_a, \rho_b)$.
\end{theorem}
The proof to Theorem \ref{thm weak consist} and the related example are presented in Appendix \ref{proof thm consistency}.
In implementation, we parametrize both the map $T$ and the dual variable $f$ by the neural networks $T_\theta, f_\eta$, with $\theta,\eta$ being the parameters of the networks. Consequently, our goal becomes solving the following saddle point problem
\begin{align}
  \max_\eta \min_\theta & ~ \mathcal{L}(T_\theta,f_\eta) :=  \label{eqn:obj}                                              \\
                        & \frac{1}{N} \sum_{k=1}^N c( {X_k}, T_\theta( {X_k})) - f_\eta(T_\theta( {X_k})) +f_\eta( {Y_k})
  \nonumber
\end{align}
where $N$ is size of the datasets
and $\{X_k\}, \{Y_k\}$ are samples drawn by $\rho_a$ and $\rho_b$ separately. The algorithm is summarized in Algorithm \ref{alg:1}. The computational complexity of our algorithm is similar with Generative Adversarial Network (GAN)-type methods. In particular, Algorithm \ref{alg:1} requires $\mathcal O(K(K_1 + K2) B )$ operations in total.

\begin{algorithm}[hb!]
  \caption{Computing
    optimal Monge map from $\rho_a$ to $\rho_b$
  }
  \begin{algorithmic}[1]%
    \State \textbf{Input}:
    Marginal distributions $\rho_a$ and $\rho_b$, Batch size $B$, Cost function $c(x,y)$.
    \State Initialize ${T}_{\theta}, {f}_{\eta}$.
    \For{$K $ steps}
    \State Sample $\{X_k\}_{k=1}^B$ $\sim \rho_a$.
    Sample $\{Y_k\}_{k=1}^B$ $\sim  \rho_b$.
    \State Update
    ${\theta}$ to {decrease} (\ref{eqn:obj}) for $K_1$ steps.
    \State Update
    ${\eta}$ to {increase} (\ref{eqn:obj}) for $K_2$ steps.
    \EndFor
    \State \textbf{Output}: The transport map $T_\theta$.
  \end{algorithmic}
  \label{alg:1}
\end{algorithm}

\paragraph{Comparison with GAN}
It is worth pointing out that our method and the Wasserstein Generative Adversarial Network (WGAN) \cite{wgan} are similar in the sense that they are both carrying out minimization over the generator/transport map and maximization over the discriminator/dual potential. However, there are two main distinctions, which are summarized in the following.
\vspace{-0.2cm}
\begin{itemize}
  \item (\textbf{Purpose}: arbitrary map vs optimal map) The purpose of WGAN is to compute for an \textit{arbitrary map} $T$ such that $T_\sharp\rho_a$ is close to the target distribution $\rho_b$; On the other hand, the purpose of our method is two-folds: we not only compute for the map $T$ that pushforwards $\rho_a$ to $\rho_b$, but also guarantee the \textit{optimality} of $T$ in the sense of minimizing the total transportation cost $\mathbb{E}_{X\sim\rho_a}c(X, T(X))$.
  \item (\textbf{Designing logic}: minimizing distance vs computing distance itself) In WGAN, one aims at minimizing the OT distance as the discrepancy between $T_\sharp\rho_a$ and the target distribution $\rho_b$, the optimal value of the corresponding loss function thus equals to 0; In contrast, our proposed method
        computes
        OT distance (as well as the transport map) between source distribution $\rho_a$ and the target $\rho_b$. In this case, the optimal value of the associated loss function equals $C_{\textrm{Monge}}(\rho_a, \rho_b)$.
\end{itemize}
We provide more detailed discussion on the comparison between our method and WGAN in Appendix \ref{rel GAN Monge}.

\section{
  Error Estimation via Duality Gaps}
In this section, we assume $m=n=d$, i.e. we consider Monge problem on $\mathbb{R}^d$. Suppose we solve \eqref{max - min} to a certain stage and obtain the pair $(T,f)$, inspired by \citet{hutter2020minimax} and \citet{icnnot}, we provide an {\it a posterior} estimate to a weighted $L^2$ error between our computed map $T$ and the optimal Monge map $T_*$. Before we present our result, we need the following assumptions
\begin{assumption}[on cost $c(\cdot, \cdot)$] \label{assmpt cost}
  We assume $c\in C^2(\mathbb{R}^d\times\mathbb{R}^d)$ is bounded from below. Furthermore, for any $x, y \in \mathbb{R}^d$, we assume $\partial_x c(x,y)$ is injective w.r.t. $y$; $\partial_{xy}c(x,y)$, as a $d\times d$ matrix, is invertible; and $\partial_{yy}c(x,y)$ is independent of $x$.
\end{assumption}

\begin{assumption}[on marginals $\rho_a$, $\rho_b$]\label{assmpt rho}
  We assume that $\rho_a,\rho_b$ are compactly supported on $\mathbb{R}^d$, and $\rho_a$ is absolutely continuous w.r.t. the Lebesgue measure.
\end{assumption}

\begin{assumption}[on dual variable $f$]\label{assmpt dual varb}
  Assume the dual variable $f\in C^2(\mathbb{R}^d)$ is always taken from $c$-concave functions, i.e., there exists certain $\varphi\in C^2(\mathbb{R}^d)$ such that $f(\cdot) = \inf_x\{\varphi(x)+c(x,\cdot)\}$ (c.f. Definition 5.7 of \cite{villani2008optimal}). Furthermore, we  assume that there exists a unique minimizer $x_y \in \textrm{argmin}_x\{\varphi(x) + c(x,y)\}$ for any $y\in \mathbb{R}^d$.
  And the Hessian of $\varphi(\cdot) + c(\cdot, y)$ at $x_y$ is positive definite.
\end{assumption}
For the sake of conciseness, let us introduce two notations. We denote $\sigma(x,y) = \sigma_{\min}(\partial_{xy}c(x,y))>0$ as the minimum singular value of $\partial_{xy}c(x,y)$; and denote $\lambda(y) = \lambda_{\max}(\nabla^2_{xx}(\varphi(x)+c(x,y))|_{x=x_y})>0$ as the maximum eigenvalue of the Hessian of $\varphi(\cdot)+c(\cdot,y)$ at $x_y$.

We denote the duality gaps as
\begin{align*}
   & \mathcal{E}_1(T, f) = \mathcal{L}(T,f) - \inf_{\widetilde{T}}\mathcal{L}(\widetilde{T},f),                                                                  \\
   & \mathcal{E}_2(f) = \sup_{\widetilde{f}} \inf_{\widetilde{T}} \mathcal{L}(\widetilde{T},\widetilde{f}) - \inf_{\widetilde{T}} \mathcal{L}(\widetilde{T},f) .\end{align*}
It is not hard to verify that the conditions mentioned in Theorem \ref{thm monge} are satisfied under Assumption \ref{assmpt cost} and Assumption \ref{assmpt rho}. Thus the Monge map $T_*$ to the Monge problem \eqref{Monge problem} exists and is unique. We now state the main theorem on error estimation:
\begin{theorem}[Posterior Error Enstimation via Duality Gaps]\label{thmerror est}
  Suppose Assumption \ref{assmpt cost}, \ref{assmpt rho} and \ref{assmpt dual varb} hold. Let us further assume the max-min problem \eqref{max - min} admits a saddle point $(\bar{f}, \bar{T})$ that is consistent with the Monge problem, i.e. $\bar{T}$ equals $T_*$, $\rho_a$ almost surely. Then there exists a strict positive weight function $\beta(x) > \underset{y\in\mathbb{R}^m}{\min}\left\{\frac{\sigma(x,y)}{2\lambda(y)}\right\}$ such that the weighted $L^2$ error between computed map $T$ and the Monge map $T_*$ is upper bounded by
  \begin{equation*}
    \|T-T_*\|_{L^2(\beta\rho_a)} \leq \sqrt{2(\mathcal{E}_1(T, f) + \mathcal{E}_2(f))}.
  \end{equation*}
  The exact formulation of $\beta$ is provided in \eqref{def beta} in the appendix
  \ref{proof thm err est}.
\end{theorem}
\begin{proof}[Sketch of proof ]
  Let us denote $\Psi_x(y) = f(y)- c(x,y)$. A
  key step of the proof is to show certain concavity of $\Psi_x(\cdot)$ given $f$ as a $c-$concave function (c.f. Assumption \ref{assmpt dual varb}). This is proved in Lemma \ref{concavity lemm} in Appendix \ref{proof thm err est}.

  Once $\Psi_x(\cdot)$ is concave, there exists unique $T_f(x)\in\textrm{argmax}_y\{\Psi_x(y)\}$.
  Thus we are able to write $\mathcal{E}_1(f, T)$ as
  \begin{align*}
    \mathcal{E}_1(T,f)  = & -\int [f(T(x)) - c(x,T(x))]\rho_a~dx                                                             \\
                          & + \sup_{\widetilde{T}} \left\{\int[f(\widetilde{T}(x)) - c(x,\widetilde{T}(x))]\rho_a~dx\right\} \\
    =                     & \int [\Psi_x(T_f(x))-\Psi_x(T(x))]\rho_a(x)~dx.
  \end{align*}
  The concavity of $\Psi_x$ then leads to
  \begin{equation}
    \mathcal{E}_1(T,f)\geq \|T_f-T\|_{L^2(\beta\rho_a)}^2,\label{E1 est}
  \end{equation}
  where $\beta(\cdot)$ depends on the concavity of $\Psi_x(\cdot)$.

  Similarly, we can rewrite $\mathcal{E}_2(f)$ as
  \begin{align*}
    \mathcal{E}_2(f) = %
    \int [\Psi_x(T_f(x)) - \Psi_x(\bar{T}(x))]\rho_a(x)~dx.
  \end{align*}
  By the similar concavity argument, one can prove $\mathcal{E}_2(T,f)\geq \|T_f-\bar{T}\|_{L^2(\beta\rho_a)}^2$. Furthermore, since $\bar{T}=T_*$, $\rho_a$ almost surely, we have the following estimation
  \begin{equation}
    \mathcal{E}_2(T,f)\geq \|T_f-T_*\|_{L^2(\beta\rho_a)}^2. \label{E2 est}
  \end{equation}
  Finally, applying the triangle inequality to \eqref{E1 est}, \eqref{E2 est} leads to the desired estimation.
\end{proof}

\begin{remark}
  We can verify that $c(x,y) = \frac{1}{2}\|x-y\|^2$ or $c(x,y)=-x\cdot y$ satisfy the conditions mentioned above. %
  More specifically, when $c(x,y)=-x\cdot y$, then $\sigma(x,y)=1$, and $\lambda(y) = \frac{1}{\lambda_{\min}(\nabla^2 f(y))}$. This recovers the upper bounds of Theorem 3.6 mentioned in \cite{icnnot}.
\end{remark}
\begin{remark}
  Suppose $c$ satisfies Assumption \ref{assmpt cost}. If $c$ is also an analytical function, then $c$ has the particular form $\Psi(x) + F(x)^\top y + \Phi(y)$, where $\Psi,u,\Phi$ are analytical functions on $\mathbb{R}^d$, and $F:\mathbb{R}^d\rightarrow\mathbb{R}^d$ is analytic with invertible Jacobian $DF(x)$ at any $x\in \mathbb{R}^d$.
\end{remark}

\section{Related work}

Discrete OT methods~\citep{grouplasso,feydy2020fast,pooladian2021entropic,meng2019large} solve the  Kantorovich formulation with EMD~\citep{nash2000dantzig} or Sinkhorn algorithm~\citep{cuturi2013sinkhorn}.
It normally computes an optimal coupling of two empirical distributions, and cannot provide out-of-sample estimates.
With an additional parameterization of a mapping function, such as linear or kernel function space~\citep{perrot2016mapping}, they can map the unseen data. However, this map is not the solution of Monge problem but only an approximation of the optimal coupling~\citep[Sec. 2.1]{courty2017joint}.
\citet{perrot2016mapping} requires solving the matrix inverse when updating the transformation map, thus introduces a risk of instability. Their function space of the map is also less expressive as neural networks, which makes it difficult to solve the OT problem with complex transport costs and high dimension data. Moreover, they can not handle large scale datasets.

With the rise of neural networks, neural OT has become popular with an advantage of dealing with continuous measure.
\citet{largeot,soptilargescaleot} solve the regularized OT problem, and as such introduce the bias.
Another branch of work~\citep{icnnot,wasserstein2gan,icnnwasserbary,korotin2021continuous} comes with parameterizing Brenier potential by Input Convex Neural Network (ICNN)~\citep{icnn}. The notable work~\citep{icnnot} is special case of our formula. In particular, their dual formula reads
\begin{align*}
  \sup_{h} \inf_{g}  -\int [\<x,\nabla g(x) - h(\nabla g(x)) \> ] \rho_adx -\int h(y) \rho_b dy,
\end{align*}
where $g,h$ are ICNN.
After a change of variable $\nabla g =T , h = -f + \|\cdot\|^2/2 $, it becomes the same as our \eqref{eq:L def} equipped with the cost $c(x,y) = \|x-y\|^2/2 $.
Later, the ICNN parameterization was shown to be less suitable~\citep{otmap5,fan2022variational,korotin2022wasserstein} for large scale problems because ICNN is not expressive enough.

Recently, several works have illustrated the scalability of our dual formula with different realizations of the transportation costs.
When $c(x, y) = \|x-y\|_2^2/2$, our method directly boils down to reversed maximum solver~\citep{ nhan2019threeplayer}, which appears to be the best neural OT solvers among multiple baselines~\citep{otmap5}.
Another variant  \citep{rout2022generative} of our work obtains comparable performance in image generative models, which asserts the efficacy in unequal dimension tasks.
\citet{gazdieva2022unpaired} utilizes the same dual formula with more diverse costs in image super-resolution task.

Based on the similar dual formula, several works propose to add random noise as an additional input to make the map stochastic, i.e. one-to-many mapping, so it can approximate the Kantorovich OT plan.
To make the learning of stochastic map valid, \citet{korotin2022neural} extends the transport cost $c(x,y)$ to be \textit{weak cost} that depends on the mapped distribution. Typically, they involve the variance in the weak cost to enforce the diversity.
The stochastic map is, however, not suitable for our problem \eqref{Monge problem} because of the conditional collapse behavior~\citep[Sec 5.1]{korotin2022neural}.
Later, \citet{asadulaev2022neural} extends the dual formula to more general cost functional. We note that the term \textit{general cost} in \citet{asadulaev2022neural} is different with \textit{general cost} in our paper. Their formula is general in a higher level and extends ours.

\section{Experiments}
In this section, we specialize \eqref{eq:L def} with different general costs $c(x,y)$ to fit in various applications.
We will not focus on the quadratic cost $c(x,y)=\|x-y\|_2^2$ since our formula in this special case has been extensively studied in multiple scenarios, e.g.  generative model~\citep{rout2022generative}, super-resolution~\citep{gazdieva2022unpaired}, and style transfer~\citep{korotin2022neural}. Instead, we focus on data with specific structure, which will be revealed in the cost function.

\subsection{Unpaired text to image generation}
We consider the task of generating images given text prompts.
Existing successful text to image generation algorithms~\citep{ramesh2021zero,ramesh2022hierarchical,saharia2022photorealistic} are supervised learning methods, and as such rely on paired data. In real world, it can be exhausting to maintain large-scale paired datasets given that they are mostly web crawling data, and the validity period of web data is limited.
We intend to learn a map between the
the text and image embedding space of the CLIP model without paired data. Our framework is shown in Figure \ref{fig:diagram}. We use the same
unpaired data generation scheme as \citet{rout2022generative}.
\begin{figure}[h]
  \centering
  \begin{subfigure}{0.48\textwidth}
    \centering
    \includegraphics[width=1\linewidth]{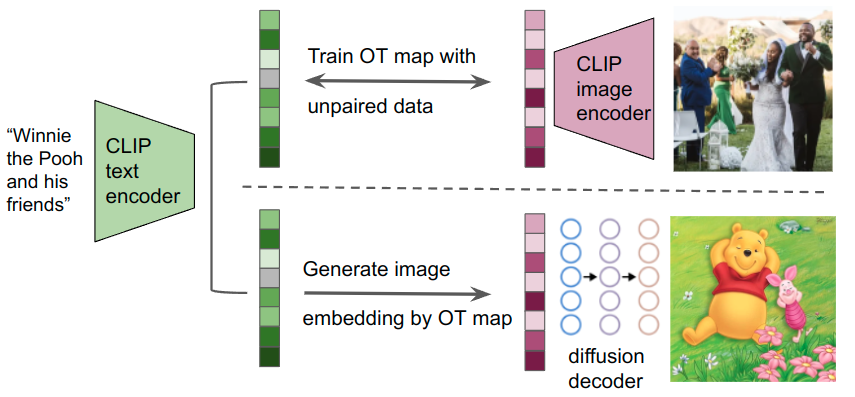}
  \end{subfigure}
  \caption{
    Our pipeline is motivated by DALL$\cdot$E2~\citep{ramesh2022hierarchical}.
    During the training (above the dotted line), our map learns to generate image embeddings that maximize the expected similarity with text embeddings. During the evaluation (below dotted line), given a text encoding (omitted from the figure) and a text embedding from the CLIP model, our map outputs an image embedding, which conditions a pretrained diffusion decoder to generate an image.
    The CLIP encoder and diffusion decoder are both pretrained and frozen.
  }
  \label{fig:diagram}
\end{figure}
\begin{figure}[h]
  \centering
  \begin{subfigure}{0.5\textwidth}
    \centering
    \includegraphics[width=1\linewidth]{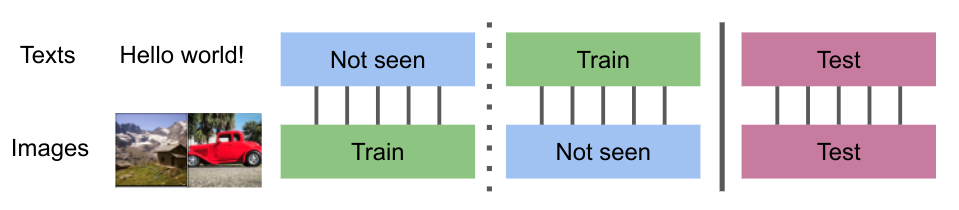}
  \end{subfigure}
  \caption{Our unpaired data generation process: we split a paired  text-image dataset training into two parts: $\mathcal S$ and $\mathcal T$, and take the texts from $\mathcal S$ and images from $\mathcal T$ as our training data. Finally, we use the paired test data so we can compare with the real images.
  }
  \label{fig:split}
\end{figure}

In our algorithm, the sample from source distribution $\rho_a$ is composed of a text encoding $Q \in\R^{77 \times 768} $ and a text embedding $x \in \R^{768}$ from CLIP~\citep{radford2021learning} model, and the sample from target distribution $\rho_b$ is the image embedding $y \in \R^{768}$.
The CLIP model is pretrained on 400M (text, image) pairs of web data with contrastive loss, such that the paired text and image embeddings share a large cosine similarity, and the non-paired embeddings present a small similarity.
As a result, we choose the transport cost as the negative cosine similarity
\begin{align*}
  c(x,y) = - \text{cos-sim} (x,y) = -\frac{\<x,y\>}{\|x\|_2 \|y\|_2},
\end{align*}
which would enforce our map to generate an image embedding relevant to the input text.

We evaluate our algorithm on two datasets: \href{https://github.com/rom1504/img2dataset/blob/main/dataset_examples/laion-art.md}{Laion art}
and \href{https://ai.google.com/research/ConceptualCaptions/download}{Conceptual captions 3M (CC-3M)}.
Laion art dataset is filtered from a 5 billion dataset to have the high aesthetic level, which is suitable for the learning of image generation, while CC-3M is not curated.
We use the OpenAI's CLIP (ViT-L/14) model and a \href{https://huggingface.co/laion/DALLE2-PyTorch/tree/main/decoder}{DALL$\cdot$E2 diffusion decoder}.
The training of their diffusion decoder requires paired data and its train dataset includes Laion art but not CC-3M. Therefore, our results on CC-3M are fully based on the unpaired data because no part in Figure \ref{fig:diagram} has seen paired data of CC-3M, while on Laion art, the diffusion decoder has some paired data information gained from pretraining.

We show qualitative samples in Figure \ref{fig:laion_art_prior} and \ref{fig:cc3m_prior}, where we compare with DALL$\cdot$E2-Laion, a DALL$\cdot$E2 model pretrained on the paired Laion aesthetic dataset, which includes Laion art dataset. Figure \ref{fig:laion_art_prior} confirms that our transport map can generate image embeddings with comparable quality even without paired data. Since CC-3M is a zero-shot dataset for DALL$\cdot$E2-Laion, the performance of DALL$\cdot$E2-Laion clearly drops in Figure \ref{fig:cc3m_prior}, while our model still generates reasonable images.

\begin{figure}[h]
  \centering
  \begin{subfigure}{0.24\textwidth}
    \centering
    \includegraphics[width=1\linewidth]{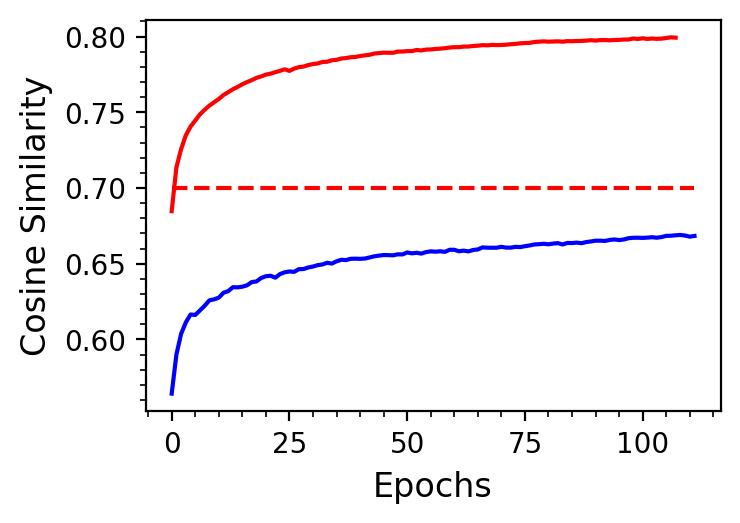}
    \caption{}
  \end{subfigure}
  \begin{subfigure}{0.235\textwidth}
    \centering
    \includegraphics[width=1\linewidth]{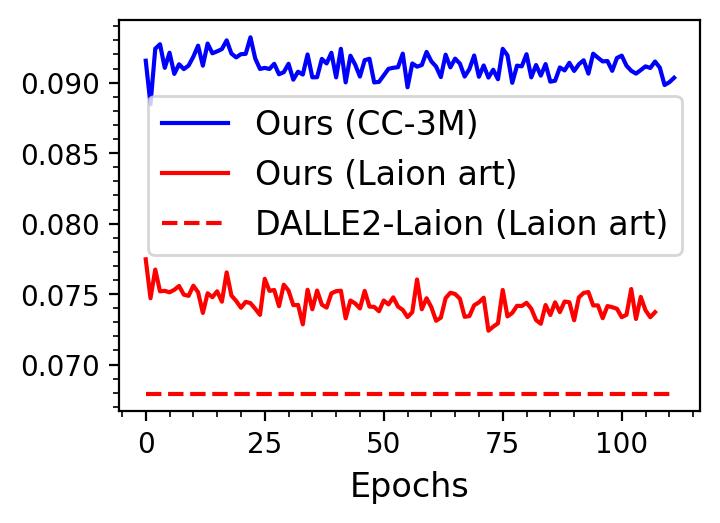}
    \caption{}
  \end{subfigure}
  \caption{
    Averaged cosine similarity between the generated image embeddings and a) the ground truth image embeddings b) the unrelated text embeddings. The similarity in the left depicts how well the model recovers the real image embedding, and the similarity in the right quantifies the overfitting behaviour (the lower the better).
  }
  \label{fig:sim}
\end{figure}
We also quantitatively compare with the baseline in terms of cosine similarity in Figure \ref{fig:sim}.
Our method achieves higher cosine similarity w.r.t. the real images on Laion art dataset. In practice, we observe that the relevance of each (text, image) pair in CC-3M is more noisy than Laion art, which makes the model more difficult to learn the real image embedding. This expalins why our cosine similarity w.r.t real image embedding on CC-3M can not converge to the same level as Laion art.  In the mean time, our overfitting level is very low on both datasets.

\subsection{Image inpainting}

In this section, we show the effectiveness of our method on the inpainting task with random rectangle masks. We take the distribution of occluded images to be $\rho_a$ and the distribution of the full images to be $\rho_b$. In many inpainting works, it's assumed that an unlimited amount of paired training data is accessible \citep{zeng2021cr}. However, most real-world applications do not involve the paired datasets. Accordingly, we consider the unpaired inpainting task, i.e. no pair of masked image and original image is accessible. The training and test data are generated in the same way as Figure \ref{fig:split}.
We choose cost function to be mean squared error (MSE) in the unmasked area
\begin{align*}
  c(x,y) = \alpha \cdot \frac{\|x \odot M - y \odot M\|_2^2}{n},
\end{align*}
where $M$ is a binary mask with the same size as the image. $M$ takes the value 1 in the unoccluded region, and 0 in the  unknown/missing region.  $\odot$ represents the point-wise multiplication, $\alpha$ is a tunable coefficient, and $n$ is dimension of $x$.
Intuitively, this works as a regularization that the pushforward images should be consistent with input images in the unmasked area. Empirically, the map learnt with a larger $\alpha$ can generate more realistic images with natural transition in the mask border and exhibit more details on the face (c.f. Figure \ref{fig:celeba64_addition} in the appendix).
\begin{figure}[ht!]
  \centering
  \begin{subfigure}{0.23\textwidth}
    \centering
    \includegraphics[width=1\linewidth]{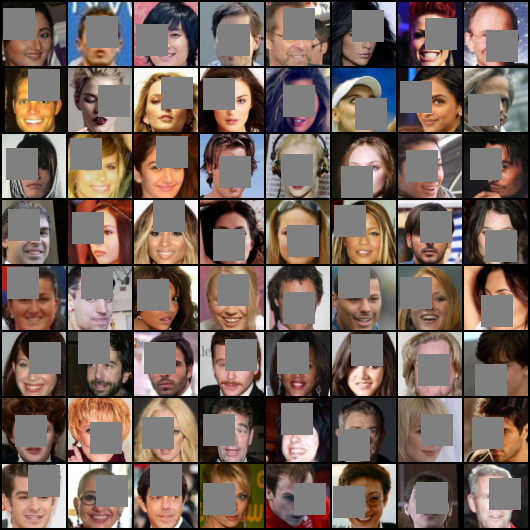}
    \includegraphics[width=1\linewidth]{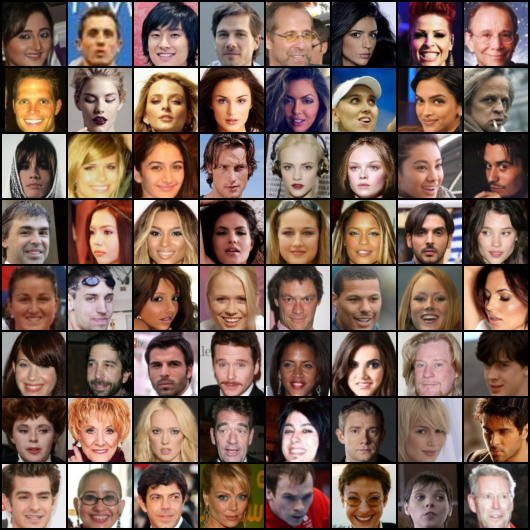}
    \caption{Masked and real images}
  \end{subfigure}
  \begin{subfigure}{0.23\textwidth}
    \centering
    \includegraphics[width=1\linewidth]{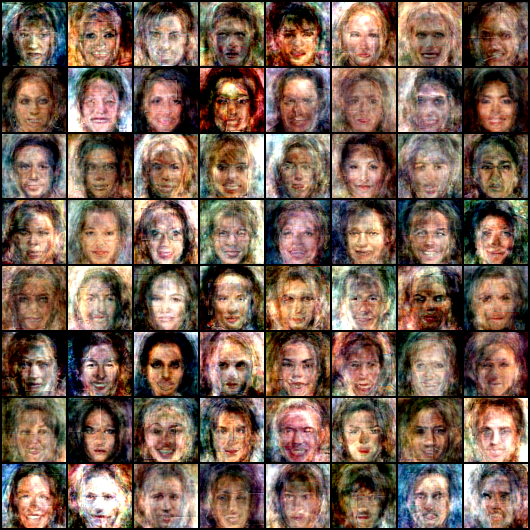}
    \includegraphics[width=1\linewidth]{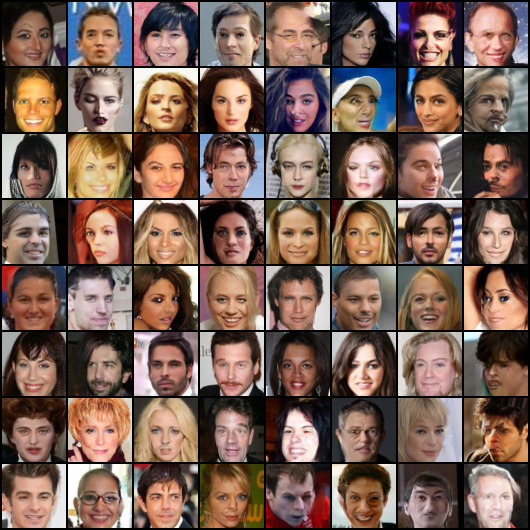}
    \caption{Pushforward images $T(x)$}
  \end{subfigure}
  \caption{Unpaired image inpainting on \textbf{test} dataset of CelebA $64\times 64$.
    In panel (b), we show the pushforward images of the discrete OT method \citep{perrot2016mapping} in the first row and ours in the second row.
  }
  \label{fig:celeba64}
\end{figure}

\begin{figure}[ht!]
  \centering
  \begin{subfigure}{0.5\textwidth}
    \centering
    \includegraphics[width=1\linewidth]{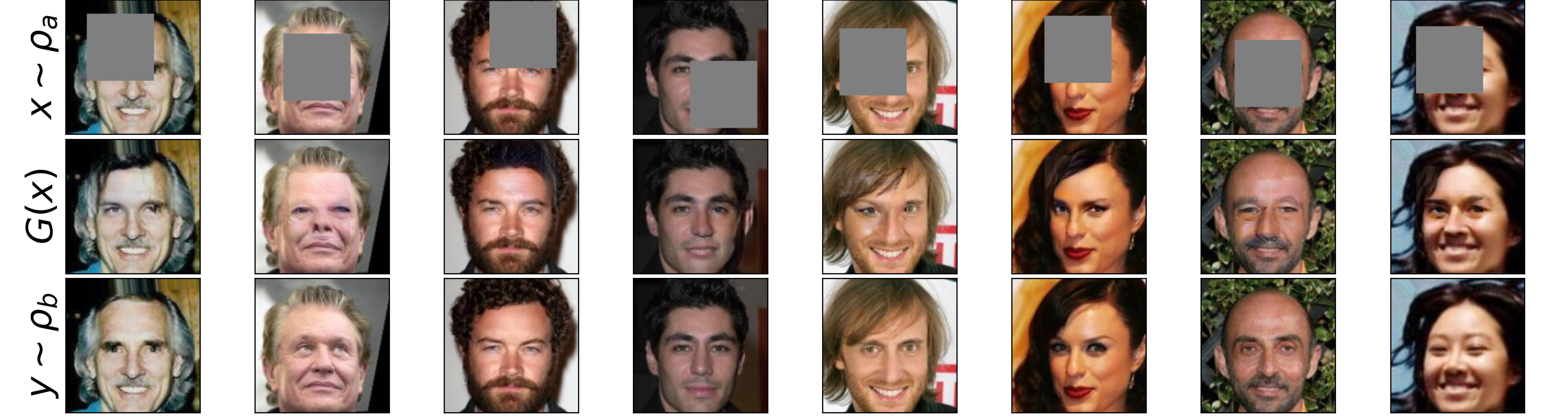}
  \end{subfigure}
  \caption{Unpaired image inpainting on \textbf{test} dataset of CelebA $128 \times 128$.
    Denote $M^C$ as the complement of $M$, i.e. $M^C=1$ in the occluded area and $0$ otherwise. We take the composite image
    $G(x)= T(x) \odot M^C + x \odot M$
    as the output image.
  }
  \label{fig:celeba128}
\end{figure}

We conduct the experiments on CelebA $64 \times 64$ and $128 \times 128$ datasets \citep{liu2015deep}. The input images ($\rho_a$) are occluded by randomly positioned square masks. Each of the source $\rho_a$ and target $\rho_b$ distributions  contains $80k$ images. We present the empirical results of inpainting in Figure \ref{fig:celeba64} and \ref{fig:celeba128}.
In Figure \ref{fig:celeba64}, we compare with the discrete OT method \citet{perrot2016mapping} since it also accepts general costs. It fits a map that approximates the barycentric projection so it can also provide out-of-sample mapping, although their map is not the solution of our Monge problem.
We use 1000 unpaired samples for \citet{perrot2016mapping} to learn the optimal coupling.
The comparison in Figure \ref{fig:celeba64} confirms that our recovered images have much better quality than the discrete OT method.

We also evaluate Fréchet Inception Distance \citep{heusel2017gans} of the generated composite images w.r.t. the original images on the test dataset. We use $40k$ images and compute the score with the implementation provided by \citet{obukhov2020torchfidelity}. We choose WGAN-GP~\citep{gulrajani2017improved} as the baseline and
show the comparison results in Table \ref{tab:fid}. It shows that the transportation cost $c(x,y)$ substantially promotes a map that generates more realistic images.

\begin{table}[h]
  \caption{Quantitative evaluation results on CelebA $64 \times 64$ test dataset.}
  \label{tab:fid}
  \begin{center}
    \setlength\tabcolsep{4pt} %
    \begin{tabular}{|c|cll|c|}
      \hline
      \multirow{2}{*}{} & \multicolumn{3}{c|}{Our method}   & \multirow{2}{*}{WGAN-GP}                                                                                    \\ \cline{2-4}
                        & \multicolumn{1}{c|}{$\alpha = 0$} & \multicolumn{1}{c|}{$\alpha = 1000$} & \multicolumn{1}{c|}{$\alpha = 10,000$} &                             \\ \hline
      FID               & \multicolumn{1}{c|}{ 18.7942}     & \multicolumn{1}{c|}{4.7621}          & \multicolumn{1}{c|}{\textbf{3.7109}}   & \multicolumn{1}{c|}{6.7479} \\ \hline
    \end{tabular}
  \end{center}
\end{table}

\vspace{-.6cm}

\subsection{Population transportation on the sphere}\label{population transportation}
Motivated by the spherical transportation example introduced in \citet[Section 4.2]{metaOT}, we consider the following population transport problem on the sphere as a synthetic example for testing our proposed method. It is well-known that the population on earth is not distributed uniformly over the land due to various factors such as landscape, climate, temperature, economy, etc. Suppose we ignore all these factors, and we would like to design an optimal transport plan under which the current population travel along the earth surface to form a rather uniform (in spherical coordinate) distribution over the earth landmass. We treat the earth as an ideal sphere with radius $1$, then we are able to formulate such problem as a Monge problem defined on $D=[0,2\pi)\times[0, \pi]$ with certain cost function $c$. To be more specific, we consider the spherical coordinate $(\theta, \phi)$ on $D$, which corresponds to the point $(\cos\theta\sin\phi, \sin\theta\sin\phi, \cos\phi)$ on the sphere. For any $(\theta_1,\phi_1), (\theta_2,\phi_2)\in D$, we set the distance $c(\cdot,\cdot)$ function as the \textit{geodesic distance} on sphere, which is formulated as
\begin{align*}
    & c((\theta_1,\phi_1), (\theta_2,\phi_2))                                    \\
  = & \arccos(\sin\phi_1\sin\phi_2\cos(\theta_1-\theta_2)+\cos\phi_1\cos\phi_2).
\end{align*}
Then assume the current population distribution over the sphere is denoted as $\rho_a$, this introduces the corresponding distribution $\rho_a^{\textrm{Sph}}$ on $D$. Denote $D_{\textrm{land}}\subset D$ as the region
of the land in the spherical coordinate system. We set the target distribution $\rho_b^{\textrm{Sph}}$ as the uniform distribution supported on $D_{\textrm{land}}$. We aim at solving the following Monge problem
\begin{equation*}
  \min_{T, T_\sharp{\rho}_a^{\textrm{Sph}}=\rho_b^{\textrm{Sph}}} \left\{ \int_D c((\theta,\phi), T(\theta, \phi))\rho_a^{\textrm{Sph}} ~d\theta d\phi \right\}.
\end{equation*}
The samples used in our experiment are generated from the licensed dataset from \cite{geodata}. In our exact implementation, in order to avoid the explosion of gradient of $\arccos(\cdot)$ near $\pm 1$, we replace the cost $c$ with its linearization $\widehat{c}=\frac{\pi}{2}-(\sin\phi_1\sin\phi_2\cos(\theta_1-\theta_2)+\cos\phi_1\cos\phi_2)$.
Furthermore, to guarantee that each sample point is transported onto landmass, we composite the trained map $T$ with a map $\tau:D\rightarrow D$ that remains samples on land unchanged but maps samples on sea back to the closest location on land among randomly selected sites.
In Figure \ref{spherical transport}, we compare our result with the
linear
transformation
method introduced in \citet{perrot2016mapping}.
For
more general kernel map, such as Gaussian kernel, their algorithm is not very stable and it is very difficult to obtain valid results.

More examples on distributions with unequal dimensions, cost of decreasing functions, and Monage map on sphere are presented in Appendix \ref{add }.

\begin{figure}[ht!]
  \centering
  \begin{subfigure}{0.45\textwidth}
    \centering
    \includegraphics[width=1\linewidth]{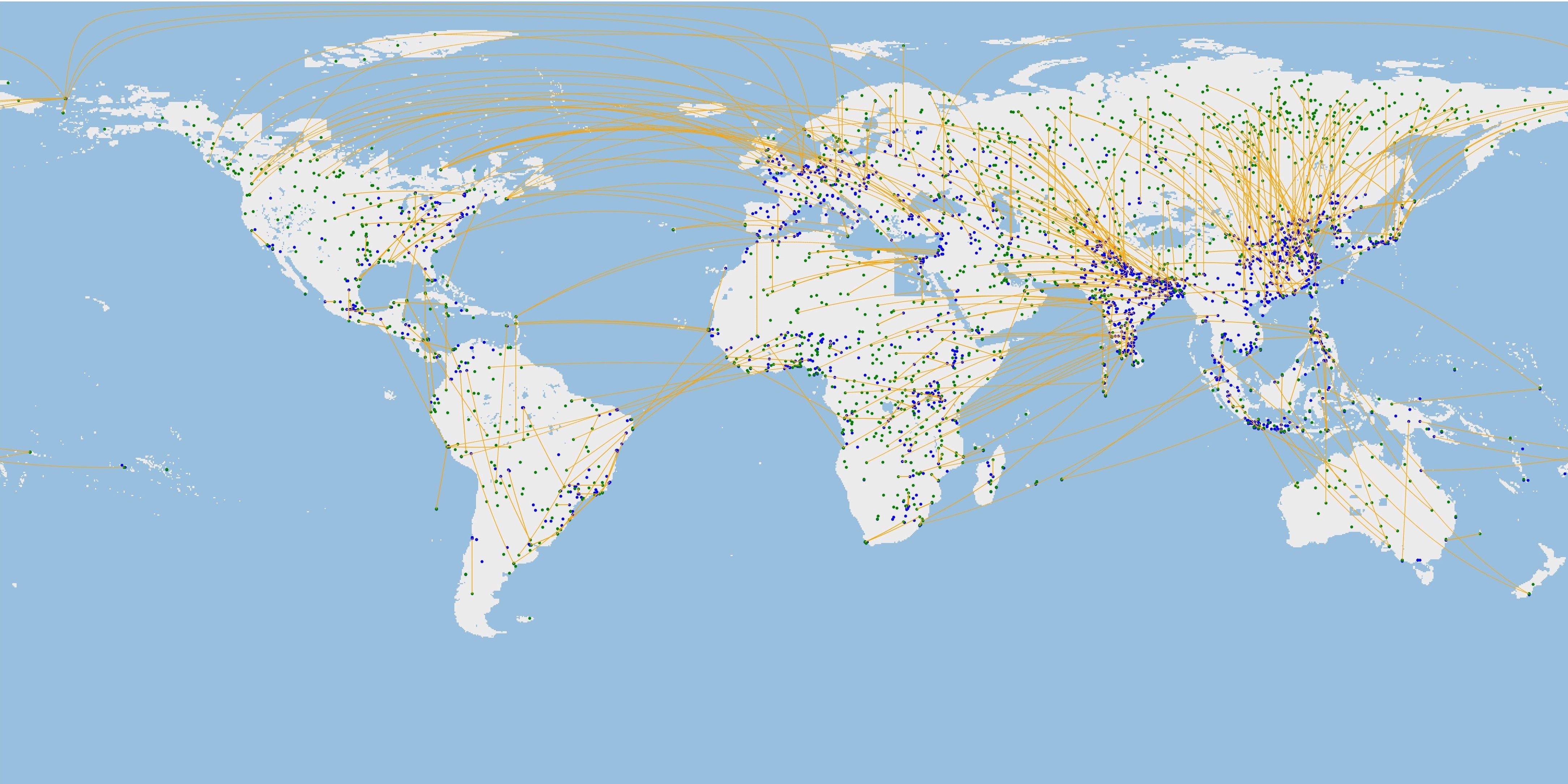}
    \vspace{-1cm}
  \end{subfigure}
  \begin{subfigure}{0.45\textwidth}
    \centering
    \includegraphics[width=1\linewidth]{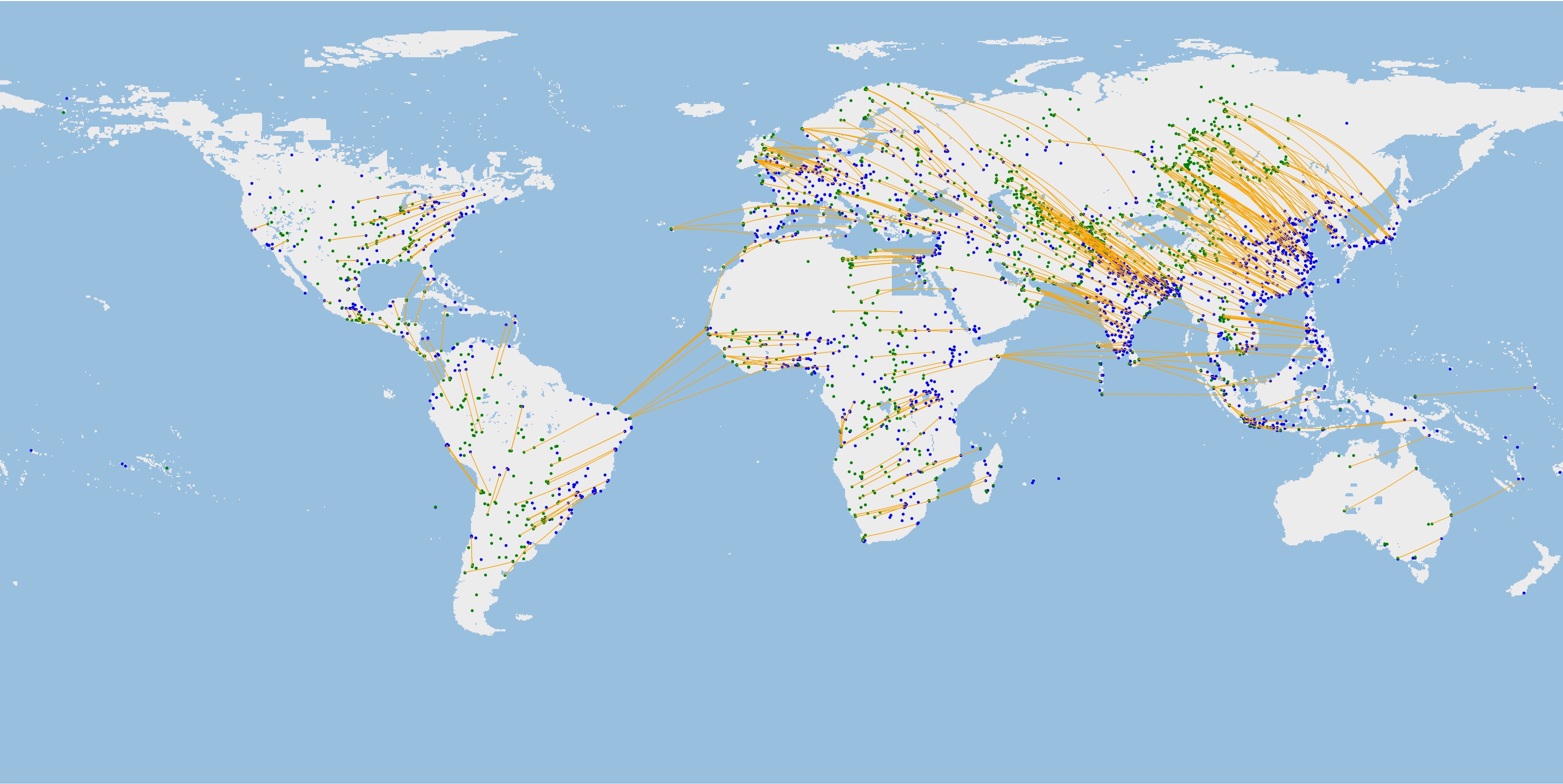}
  \end{subfigure}
  \caption{ We show our result on top and \citet{perrot2016mapping} in the bottom.
    Each figure plots samples from the source distribution $\rho^{\textrm{Sph}}_a$ (blue) and samples from the pushforwarded distribution $T_{\theta\sharp}\rho^{\textrm{Sph}}_a$ (green).
    We also demonstrates the computed transport map (orange)
    for the first $500$ randomly generated points from source $\rho_a^{\textrm{Sph}}$ to target $\rho_b^{\textrm{Sph}}$.}
  \label{spherical transport}
\end{figure}

\vspace{-0.5cm}

\section{Conclusion and limitation}
In this paper we present a novel method to compute Monge map between two given distributions with
flexible
transport cost functions.
In particular, we consider applying Lagrange multipliers on the Monge problem, which leads to a max-min saddle point problem.
By further introducing neural networks into our optimization,
we obtain a scalable algorithm that can handle most general costs and even the case where the dimensions of marginals are unequal.
Our scheme is shown to be effective through a series of experiments
with
high dimensional large scale datasets, where we only assume the unpaired samples from marginals are accessible.
It will become an useful tool for machine learning applications such as domain adaption, and image restoration that require transforming data distributions.
In the future, it is promising to specialize our method to structured data, such as time series, graphs, and point clouds.
\paragraph{Limitation}
We observe that when the target distribution is a discrete distribution that is with fixed support, our method tends to be more unstable. We conjecture  this is because the network parameterization of discriminator is too complex for this type of uncomplicated target distribution.

\clearpage
\bibliography{monge}

\clearpage
\appendix

\thispagestyle{empty}

\onecolumn

\section{Backgrounds of optimal transport problems}\label{background}

In this work, we assume that the cost function $c$ satisfies the following three conditions
\begin{align}
   & c \textrm{ is bounded from below, i.e., there exists } \ubar{c}\in\mathbb{R}, \textrm{ such that } c\geq \ubar{c} \textrm{ on } \mathbb{R}^n\times\mathbb{R}^m; %
  \label{condition c low bdd}                                                                                                                                        \\
   & c ~~ \textrm{is \textit{locally Lipschitz} and \textit{superdifferentiable} everywhere}; \label{condition c a}                                                  \\
   & \partial_x c(x,\cdot) ~~ \textrm{is injective for any} ~ x\in\mathbb{R}^n.\label{condition c b}
\end{align}
Here we define the local Lipschitz property and superdifferentiablity as follows.

\begin{definition}[Locally Lipschitz]
  Let $U\subset\mathbb{R}^n$ be open and let $f:\mathbb{R}^n\rightarrow \mathbb{R}$ be given. Then

  (1) $f$ is Lipschitz if there exists $L<\infty$ such that
  \begin{equation*}
    \forall x,z\in \mathbb{R}^n, \quad |f(z)-f(x)| \leq L|x-z|.
  \end{equation*}

  (2) $f$ is said to be locally Lipschitz if for any $x_0\in \mathbb{R}^n$, there is a neighbourhood $O$ of $x_0$ in which $f$ is Lipschitz.

\end{definition}

\begin{definition}[Superdifferentiablity]
  For function $f:\mathbb{R}^n\rightarrow \mathbb{R}$, we say $f$ is superdifferentiable at $x$, if there exists $p\in\mathbb{R}^n$, such that
  \begin{equation*}
    f(z)\geq f(x)+\langle p, z-x \rangle + o(|z-x|).
  \end{equation*}
\end{definition}

Recall that the Kantorovich dual problem of the primal OT problem \eqref{ g OT} is formulated as
\begin{equation*}
  K(\rho_a, \rho_b) = \sup_{\substack{\text{$(\psi,\phi)\in L^1(\rho_a)\times L^1(\rho_b)$} \\ \text{$\phi(y)-\psi(x)\leq c(x,y)~\forall~x\in\mathbb{R}^n,y\in\mathbb{R}^m$}}} \left\{\int_{\mathbb{R}^m} \phi(y)\rho_b(y)~dy - \int_{\mathbb{R}^n} \psi(x)\rho_a(x)~dx \right\}.
\end{equation*}
We denote the optimal value of this dual problem as $K(\rho_a, \rho_b)$.

It is not hard to tell that the above dual problem is equivalent to any of the following two problems
\begin{equation}
  \sup_{\psi\in L^1(\rho_a)} \left\{\int_{\mathbb{R}^m} \psi^{c,+}(y)\rho_b(y)~dy - \int_{\mathbb{R}^n} \psi(x)\rho_a(x)~dx \right\},  \label{Kantorovich formula 2 }
\end{equation}
\begin{equation}
  \sup_{\phi\in L^1(\rho_b)} \left\{\int_{\mathbb{R}^m} \phi(y)\rho_b(y)~dy - \int_{\mathbb{R}^n} \phi^{c,-}(x)\rho_a(x) ~dx \right\} . \label{Kantorovich formula 3 }
\end{equation}
Here the $c-$transform $\psi^{c,+}$, $\phi^{c,-}$ are defined via infimum/supremum convolution as:  $\psi^{c,+}(y) = \inf_{x}(\psi(x) + c(x,y))$ and $\phi^{c,-}(x) = \sup_y(\phi(y)-c(x,y))$. If we denote $\psi_*, \phi_*$ as the optimal solutions to \eqref{Kantorovich formula 2 }, \eqref{Kantorovich formula 3 } respectively. Then both $(\psi_*, \psi_*^{c, +})$ and $(\phi^{c, -}_*, \phi_*)$ are the optimal solutions to \eqref{Kantorovich formula 1 }; $\psi_*^{c,+}$, $\phi_*^{c,-}$ are also solutions to \eqref{Kantorovich formula 3 }, \eqref{Kantorovich formula 2 } respectively.

The following theorem states the equivalent relationship between the primal OT problem \eqref{ g OT} and its Kantorovich dual \eqref{Kantorovich formula 1 }.
The proof for a more general version can be found in Theorem 5.10 of \cite{villani2008optimal}.
\begin{theorem}[Kantorovich Duality]\label{thm: kan dual}
  Suppose $c$ is a cost function defined on $\mathbb{R}^n\times\mathbb{R}^m$ and satisfies  \eqref{condition c low bdd}. Recall that $C(\rho_a,\rho_b)$ denotes the infimum value of \eqref{ g OT} and $K(\rho_a,\rho_b)$ denotes the maximum value of \eqref{Kantorovich formula 1 } (or equivalently, \eqref{Kantorovich formula 2 }, \eqref{Kantorovich formula 3 }). Then $C(\rho_a,\rho_b) = K(\rho_a,\rho_b)$.
\end{theorem}

The following result states the existence and uniqueness of the optimal solution to the Monge problem. It also reveals the relation between the optimal map $T_*$ of \eqref{Monge problem} and the optimal transport plan $\pi_*$ of \eqref{ g OT}. It is a simplified version of Theorem 10.28 combined with Remark 10.33 taken from \cite{villani2008optimal}.

\begin{customthm}{1}[Existence, uniqueness and characterization of  the optimal Monge map]
  Suppose the cost $c$ satisfies \eqref{condition c low bdd}, \eqref{condition c a}, \eqref{condition c b}, we further assume that $\rho_a$ and $\rho_b$ are compactly supported and $\rho_a$ is absolute continuous with respect to the Lebesgue measure on $\mathbb{R}^n$. Then there exists unique-in-law transport map $T_*$ solving the Monge problem \eqref{Monge problem}. And the joint distribution $(\textrm{Id}, T_*)_\sharp\rho_a$ on $\mathbb{R}^{n+m}$ is the optimal transport plan of the general OT problem \eqref{ g OT}. Furthermore, there exists $\psi_*$ differentiable $\rho_a$ almost surely\footnote{$\psi_*$ is differentiable at $x$ for all $x\in\textrm{Spt}(\rho_a)\textbackslash E_0$, where $E_0$ is a zero measure set.}, and $\nabla\psi_*(x) + \partial_x c(x,T_*(x)) = 0$, $\rho_a$ almost surely.

\end{customthm}

\section{Relation between our method and generative adversarial networks}\label{rel GAN Monge}
It is worth pointing out that our scheme and Wasserstein Generative Adversarial Networks (WGAN) \cite{wgan} are similar in the sense that they are both doing minimization over the generator/map and maximization over the discriminator/dual potential. However, there are two main distinctions between them. Such differences are not reflected from the superficial aspects such as the choice of reference distributions $\rho_a$, but come from the fundamental logic hidden behind the algorithms.
\begin{itemize}
  \item We want to first emphasize that the  mechanisms of two algorithms are different:
        Typical Wasserstein GANs (WGAN) are usually  formulated as
        \begin{equation}
          \min_{G} \underbrace{\max_{\|D\|_{\textrm{Lip}\leq 1}} \int D(y)\rho_b(y) dy - \int D(G(x))\rho_a(x) dx}_{1-\textrm{Wasserstein distance } W_1(G_\sharp \rho_a,\rho_b) } \label{diff Monge Gan - Gan}
        \end{equation}
        and ours reads
        \begin{equation}
          \underbrace{\max_{f} \min_{T} \int f(y)\rho_b(y) dy - \int  f(T(x))\rho_a(x) dx +  \int  c(X,T(x))\rho_a(x) dx}_{\textrm{general Wasserstein distance } C_{\textrm{Monge}}(\rho_a,\rho_b)}  \label{diff Monge Gan - Monge}
        \end{equation}
        The inner maximization of \eqref{diff Monge Gan - Gan} computes $W_1$ distance via Kantorovich duality and the outer loop minimize the $W_1$ gap between desired $\rho_b$ and $G_\sharp \rho_a$; However, the logic behind our scheme \eqref{diff Monge Gan - Monge} is different: the inner optimization computes for the $c-$transform of $f$, i.e. $f^{c,-}(x) = \sup_{\xi}(f(\xi)-c(x,\xi))$; And the outer maximization computes for the Kantorovich dual problem $C(\rho_a,\rho_b)=\sup_f\left\{\int f(y)\rho_b(y) dy - \int f^{c,-}(x)\rho_a(x)dx\right\}$.

        Even under $W_1$ circumstance, one can verify the intrinsic difference between two proposed methods: when setting the cost $c(x,y)=\|x-y\|$,
        and $\rho_a=G\sharp \rho_a$ in \eqref{diff Monge Gan - Monge},
        the entire "max-min" optimization of \eqref{diff Monge Gan - Monge} (underbraced part) is equivalent to the inner maximization problem of \eqref{diff Monge Gan - Gan} (underbraced part), but not for the entire saddle point scheme.

        It is also important to note that WGAN aims to minimize the distance between generated distribution and the target distribution and the ideal value for \eqref{diff Monge Gan - Gan} is $0$. On the other hand, one of our goal is to estimate the optimal transport distance between the initial distribution $\rho_a$ and the target distribution $\rho_b$. Thus the ideal value for \eqref{diff Monge Gan - Monge} should be $C(\rho_a,\rho_b)$, which is not $0$ in most of the cases.

  \item  We then argue about the optimality of the computed map $G$ and $T$: In \eqref{diff Monge Gan - Gan}, one is trying to obtain a map $G$ by minimizing $W_1(\rho_b, G_\sharp \rho_a)$ w.r.t. $G$, and hopefully, $G_\sharp\rho_a$ can approximate $\rho_b$ well. However, there isn't any restriction exerted on $G$, thus one can not expect the computed $G$ to be the optimal transport map between $\rho_a$ and $\rho_b$; On the other hand, in \eqref{diff Monge Gan - Monge}, we not only compute $T$ such that $T_\sharp \rho_a$ approximates $\rho_b$ , but also compute for the optimal $T$ that minimizes the transport cost $\mathbb{E}_{\rho_a } [c(X,T(X))]$. In \eqref{diff Monge Gan - Monge}, the computation of $T$ is naturally incorporated in the max-min scheme and there exists theoretical result (recall Theorem \ref{thm consistent} in the paper) that guarantees $T$ to be the optimal transport map.
\end{itemize}
In summary, even though the formulation of both algorithms are similar, the designing logic (minimizing distance vs computing distance itself) and the purposes (computing arbitary pushforward map vs computing the optimal map) of the two methods are distinct. Thus the theoretical and empirical study of GANs cannot be trivially translated to proposed method. In addition to the above discussions, we should also refer the readers to \cite{gazdieva2022unpaired}, in which a comparison between a similar saddle point method and the regularized GANs are made in section 6.2 and summarized in Table 1.

\section{Proof of
  Consistency
 }\label{proof thm consistency}
We firstly prove the following result on the consistency of our proposed method.
\begin{customthm}{2}[Consistency with $\bar{T}_\sharp\rho_a=\rho_b$]
  Suppose the max-min problem \eqref{max - min} admits at least one saddle point solution. We denote it as $(\bar{T},\bar{f})$. Under the assumption $\bar{T}_{\sharp} \rho_a=\rho_b$, we have
  \begin{itemize}
    \item $\bar{f}$ is an optimal solution $\phi$ to the Kantorovich dual problem \eqref{Kantorovich formula 1 }. Or equivalently, $\bar{f}$ is an optimal solution to \eqref{Kantorovich formula 3 };
    \item The optimal solution to the Monge problem \eqref{Monge problem} exists, and $\bar{T}$ is an optimal solution, i.e. $\bar{T}$ is the Monge map.
    \item Furthermore, $\mathcal{L}(\bar{T}, \bar{f})=C_{\textrm{Monge}}(\rho_a, \rho_b)$.
  \end{itemize}
\end{customthm}
\begin{proof}[Proof of Theorem \ref{thm consistent}]
  First one can verify the following
  \begin{align}
    \inf_T \mathcal{L}(T,f) = & -\int_{\mathbb{R}^n}  \sup_\xi\{ f(\xi)-c(x,\xi) \} \rho_a(x) dx + \int_{\mathbb{R}^m}f(y)\rho_b(y)dy \nonumber \\
    =                         & \int_{\mathbb{R}^m} f(y)\rho_b(y)dy - \int_{\mathbb{R}^n} f^{c,-}(x)\rho_a(x)dx,\label{compute}
  \end{align}
  Then the max-min problem $\sup_f\inf_T \mathcal{L}(T, f)$ can be formulated as
  \begin{equation*}
    \sup_f \left\{ \int_{\mathbb{R}^m} f(y)\rho_b(y)dy - \int_{\mathbb{R}^n} f^{c, -}(x)\rho_a(x)dx \right\}.
  \end{equation*}
  This is exactly the Kantorovich dual problem \eqref{Kantorovich formula 3 }. Since $(\bar{T}, \bar{f})$ is the saddle point, $\bar{f}$ is an optimal solution to \eqref{Kantorovich formula 3 }. This verifies the first assertion of the theorem.

  On the other hand, at the saddle point  $(\bar{T}, \bar{f})$, we have
  \begin{equation*}
    \quad \bar{T}(x) \in \textrm{argmax}_{\xi\in\mathbb{R}^m}\{\bar{f}(\xi)-c(x,\xi)\}, ~\rho_a ~\textrm{almost surely}.
  \end{equation*}
  This leads to
  \begin{equation*}
    f^{c,-}_*(x) = \bar{f}(\bar{T}(x))-c(x,\bar{T}(x)), ~\rho_a ~\textrm{almost surely}.
  \end{equation*}
  Then we have
  \begin{align*}
    \int_{\mathbb{R}^n} c(x,\bar{T}(x))\rho_a(x)~dx & = \int_{\mathbb{R}^n} \bar{f}(\bar{T}(x))\rho_a(x)~dx - \int_{\mathbb{R}^n} \bar{f}^{c,-}(x)\rho_a(x)~dx                                \\
                                                    & = \int_{\mathbb{R}^m} \bar{f}(y)\rho_b(y)~dy - \int_{\mathbb{R}^n} \bar{f}^{c,-}(x)\rho_a(x)~dx                                         \\
                                                    & = \int_{\mathbb{R}^n\times\mathbb{R}^m} [\bar{f}(y) - f^{c,-}_*(x)]d\pi(x,y) \leq \int_{\mathbb{R}^n\times\mathbb{R}^m} c(x,y)d\pi(x,y)
  \end{align*}
  for any $\pi\in\Pi(\rho_a,\rho_b)$. Here the second equality is due to the assumption $T_{*\sharp}\rho_a = \rho_b$. The last inequality is due to the definition of $f^{c,-}_*(x)=\sup_{y}\{\bar{f}(y) - c(x,y)\}$.

  We now take the infimum value of $\int_{\mathbb{R}^n\times\mathbb{R}^m} cd\pi$ and we obtain
  \begin{equation}
    \int_{\mathbb{R}^n} c(x,\bar{T}(x))\rho_a(x)~dx \leq C(\rho_a,\rho_b),\label{consistency equ 1}
  \end{equation}
  Now, for any transport map $T$ satisfying $T_\sharp\rho_a = \rho_b$, by denoting $\pi = (\textrm{Id}, T)_{\sharp}\rho_a$, we have
  \begin{equation}
    \int_{\mathbb{R}^n} c(x, T(x)) \rho_a(x)~dx = \int_{\mathbb{R}^n\times \mathbb{R}^m} c(x,y)d\pi(x,y) \geq C(\rho_a,\rho_b). \label{consistency equ 2}
  \end{equation}
  Combining \eqref{consistency equ 1} and \eqref{consistency equ 2}, we obtain
  \begin{equation*}
    \int_{\mathbb{R}^n} c(x, \bar{T}(x))\rho_a(x)dx \leq \int_{\mathbb{R}^n} c(x, T(x))\rho_a(x)dx, \quad \textrm{for any} ~ T, ~ T_\sharp \rho_a = \rho_b.
  \end{equation*}
  This indicates the existence of the Monge map, and $\bar{T}$ is the Monge map.

  At last, we have
  \begin{equation*}
    \mathcal{L}(\bar{T}, \bar{f}) = \int_{\mathbb{R}^n} c(x, \bar{T}(x))\rho_a(x)dx + \int_{\mathbb{R}^m} \bar{f}(y)(T_{*\sharp}\rho_a - \rho_b)dy =\int_{\mathbb{R}^n}c(x, \bar{T}(x))\rho_a(x)~dx = C_{\textrm{Monge}}(\rho_a, \rho_b).
  \end{equation*}
\end{proof}

Without the assumption $\bar{T}_\sharp\rho_a=\rho_b$, we no longer have guarantee on the existence of Monge map $T_*$, not to mention the consistency between $\bar{T}$ and $T_*$. However, we are still able to show the consistency between the saddle point value $\mathcal{L}(\bar{T}, \bar{f})$ and general OT distance $C(\rho_a,\rho_b)$.
\begin{customthm}{3}[Consistency without $\bar{T}_\sharp\rho_a=\rho_b$]
\end{customthm}
\begin{proof}
  The proof of the first part of this theorem is the same as Theorem \ref{thm consistent}. Thus it is not hard to verify
  \begin{equation}
    \sup_f\int_T \mathcal{L}(T, f) = \sup_f\left\{\int_{\mathbb{R}^m} f(y)\rho_b(y)dy - \int_{\mathbb{R}^n} f^{c,-}(x)\rho_a(x)dx \right\} = K(\rho_a,\rho_b).
  \end{equation}
  Recall Theorem \ref{thm: kan dual}, the optimal value of Kantorovich dual problem equals general OT distance $C(\rho_a,\rho_b)$. This proves our assertion.
\end{proof}

\begin{remark}
  Theorem \ref{thm weak consist} indicates that although for some  general cases in which our proposed method \eqref{max - min} fails to compute for a valid transport map $\bar{T}$, \eqref{max - min} is still able to recover the exact optimal transport distance $C(\rho_a, \rho_b)$.

  More specifically, let us consider the OT problem on $\mathbb{R}$ with $c(x,y)=|x-y|^2$ and $\rho_a=\delta_0$ (point distribution at $0$), $\rho_b=\mathcal{N}(0, 1)$ (normal distribution). Our method yields
  \begin{equation*}
    \sup_f\inf_T \left\{ |T(0)|^2 + \int f(y)\rho_b(y)dy - f(T(0)) \right\}.
  \end{equation*}
  By setting $\Psi_0(y) = |y|^2-f(y)$, our max-min problem becomes
  \begin{equation*}
    \sup_f \inf_T \{ \Psi_0(T(0)) + \int (|y|^2 - \Psi_0(y))\rho_b(y)dy  \} = \underbrace{\int |y|^2\rho_b(y)dy}_{=1} + \sup_{\Psi_0} \underbrace{\left\{\int [\inf \Psi_0 - \Psi_0(y)]\rho_b(y)dy\right\}}_{\leq 0} = 1
  \end{equation*}
  The supreme over $\Psi_0$ is obtained when $\Psi_0=\textrm{Const}$, i.e., when $f(y)=|y|^2+\textrm{Const}$.
  Thus we have $\sup_f\inf_T\mathcal{L}(T, f)=1=C(\rho_a,\rho_b)$. In such example, although the Monge map does not exist, the proposed method can still capture the exact OT distance.
\end{remark}

\section{Proof of Posterior Error Estimation}\label{proof thm err est}

Let us first recall the assumptions on the cost $c(\cdot,\cdot)$, marginals $\rho_a,\rho_b$, and dual variable $f$:

\begin{assumption}[on cost $c(\cdot, \cdot)$]
  We assume $c\in C^2(\mathbb{R}^d\times\mathbb{R}^d)$ is bounded from below. Furthermore, for any $x, y \in \mathbb{R}^d$, we assume $\partial_x c(x,y)$ is injective w.r.t. $y$; $\partial_{xy}c(x,y)$, as a $d\times d$ matrix, is invertible; and $\partial_{yy}c(x,y)$ is independent of $x$.
\end{assumption}

\begin{assumption}[on marginals $\rho_a$, $\rho_b$]
  We assume that $\rho_a,\rho_b$ are compactly supported on $\mathbb{R}^d$, and $\rho_a$ is absolutely continuous w.r.t. the Lebesgue measure.
\end{assumption}

\begin{assumption}[on dual variable $f$]
  Assume the dual variable $f\in C^2(\mathbb{R}^d)$ is always taken from $c$-concave functions, i.e., there exists certain $\varphi\in C^2(\mathbb{R}^d)$ such that $f(\cdot) = \inf_x\{\varphi(x)+c(x,\cdot)\}$ (c.f. Definition 5.7 of \cite{villani2008optimal}). Furthermore, we  assume that there exists a unique minimizer $x_y \in \textrm{argmin}_x\{\varphi(x) + c(x,y)\}$ for any $y\in \mathbb{R}^d$.
  And the Hessian of $\varphi(\cdot) + c(\cdot, y)$ at $x_y$ is positive definite.
\end{assumption}

We then introduce the following two notations. We denote
\begin{equation*}
  \sigma(x,y) = \sigma_{\min}(\partial_{xy}c(x,y))>0
\end{equation*}
as the minimum singular value of $\partial_{xy}c(x,y)$; and
\begin{equation*}
  \lambda(y) = \lambda_{\max}(\nabla^2_{xx}(\varphi(x)+c(x,y))|_{x=x_y})>0
\end{equation*}
as the maximum eigenvalue of the Hessian of $\varphi(\cdot)+c(\cdot,y)$ at $x_y$.

We denote the duality gaps as
\begin{align*}
   & \mathcal{E}_1(T, f) = \mathcal{L}(T,f) - \inf_{\widetilde{T}}\mathcal{L}(\widetilde{T},f),                                                                  \\
   & \mathcal{E}_2(f) = \sup_{\widetilde{f}} \inf_{\widetilde{T}} \mathcal{L}(\widetilde{T},\widetilde{f}) - \inf_{\widetilde{T}} \mathcal{L}(\widetilde{T},f) .
\end{align*}

It is not hard to verify the following lemma:
\begin{lemma}[Existence and uniqueness of Monge map under Assumption \ref{assmpt cost}, \ref{assmpt rho}]\label{lemma: unique monge map}
  Suppose Assumption \ref{assmpt cost} and \ref{assmpt rho} hold, then the conditions on $c(\cdot, \cdot)$ and $\rho_a, \rho_b$ mentioned in Theorem \ref{thm monge} are all satisfied. Thus the Monge map $T_*$ exists and is unique in law.
\end{lemma}

We now state the main theorem on error estimation:
\begin{customthm}{4}[Posterior Error Enstimation via Duality Gaps]
  Suppose Assumption \ref{assmpt cost}, \ref{assmpt rho} and \ref{assmpt dual varb} hold. Let us further assume the max-min problem \eqref{max - min} admits a saddle point $(\bar{f}, \bar{T})$ that is consistent with the Monge problem, i.e. $\bar{T}$ equals $T_*$, $\rho_a$ almost surely. Then there exists a strict positive weight function $\beta(x) > \underset{y\in\mathbb{R}^m}{\min}\left\{\frac{\sigma(x,y)}{2\lambda(y)}\right\}$ such that the weighted $L^2$ error between computed map $T$ and the Monge map $T_*$ is upper bounded by
  \begin{equation*}
    \|T-T_*\|_{L^2(\beta\rho_a)} \leq \sqrt{2(\mathcal{E}_1(T, f) + \mathcal{E}_2(f))}.
  \end{equation*}
\end{customthm}

To prove Theorem \ref{thmerror est}, we need the following two lemmas:
\begin{lemma}\label{thm matrix pos def}
  Suppose $n\times n$ matrix $A$ is invertible with minimum singular value $\sigma_{\min}(A)>0$. Also assume $n\times n$ matrix $H$ is self-adjoint and satisfies $\lambda I_n \succeq H\succ  O_n $\footnote{Here matrix $M_1 \succ M_2$ iff $M_2 - M_1$ is a positive-definite matrix, and $M_1 \succeq M_2  $ iff $M_1 - M_2$ is a positive-semidefinite matrix.}. Then $A^\top H^{-1}A  \succeq \frac{\sigma_{\min}(A)^2}{\lambda}I_n$.
\end{lemma}
\begin{proof}[Proof of Lemma \ref{thm matrix pos def} ]
  One can first verify that $H^{-1}\succeq\frac{1}{\lambda}I_n$ by digonalizing $H^{-1}$. To prove this lemma, we only need to verify that for arbitrary $v\in\mathbb{R}^n$,
  \begin{align*}
    v^\top A^\top H^{-1}Av = (Av)^\top H^{-1}Av  \geq \frac{|Av|^2}{\lambda} \geq \frac{\sigma_{\min}(A)^2}{\lambda}|v|^2
  \end{align*}
  Thus $A^\top H^{-1}A-\frac{\sigma_{\min}(A)^2}{\lambda}I_n$ is positive-semidefinite.
\end{proof}

The following lemma is crucial for proving our results, it  analyzes the concavity of the target function $f(\cdot)-c(\cdot,y)$ with $f$ is $c$-concave.
\begin{lemma}[Concavity of $f(\cdot)-c(x,\cdot)$ if $f$ $c$-concave]\label{concavity lemm}
  Suppose the cost function $c(x,y)$ and $f$ satisfy the conditions mentioned in Theorem \ref{thmerror est}. Denote the function $\Psi_x(y) = f(y)- c(x,y)$, keep all notations defined in Theorem \ref{thmerror est},  then we have
  \begin{equation*}
    \nabla^2\Psi_x(y)  \preceq  -\frac{\sigma(x,y)^2}{\lambda(y)}I_n.
  \end{equation*}
\end{lemma}
\begin{proof}[Proof of Lemma \ref{concavity lemm}]

  First, we notice that $f$ is $c$-convex, thus, there exists $\varphi$ such that $f(y) = \inf_x\{\varphi(x)+c(x,y)\}$. Let us also denote $\Phi(x,y)=\varphi(x)+c(x,y)$.

  Now for a fixed $y\in\mathbb{R}^n$, We pick one
  \begin{equation*}
    x_y \in \textrm{argmin}_{x}\left\{ \varphi(x) +c(x,y) \right\}
  \end{equation*}
  Since we assumed that $\varphi\in C^2(\mathbb{R}^n)$ and $c\in C^2(\mathbb{R}^n\times\mathbb{R}^n)$, we have
  \begin{equation}
    \partial_x \Phi(x_y, y) = \nabla\varphi(x_y) + \partial_x c(x_y,y) = 0  \label{formu_1}
  \end{equation}
  Now recall Assumption \ref{assmpt dual varb}, %
  $\partial_{xx}^2\Phi(x_y, y)$ is positive definite, thus is also invertible. We can now apply the implicit function theorem to show that the equation $\partial_x\Phi(x,y)=0$ determines an implicit function $x(\cdot)$, which satisfies $x(y)=x_y$ in a small neighbourhood $U\subset\mathbb{R}^n$ containing $y$. Furthermore, one can show that $x(\cdot)$ is continuously differentiable at $y$. We will denote $x_y$ as $x(y)$ in our following discussion.

  Now differentiating \eqref{formu_1} with respect to $y$ yields
  \begin{align}
     & \partial_{xx}^2\Phi(x(y), y)\nabla x(y) + \partial_{xy}^2 c(x(y),y) = 0 \label{formu_2}
  \end{align}
  On one hand, \eqref{formu_2} tells us
  \begin{equation}
    \nabla x(y) = -\partial_{xx}\Phi(x(y), y)^{-1}\partial_{xy}c(x(y), y).  \label{implicit df}
  \end{equation}
  Now we directly compute
  \begin{equation}
    \nabla^2\Psi_x(y) = \nabla^2 f(y) - \partial_{yy}^2c(x,y). \label{hessian Psi_x}
  \end{equation}
  in order to compute $\nabla^2 f(y)$, we first compute $\nabla f(y)$
  \begin{equation}
    \nabla f(y) = \nabla (\varphi(x(y)) + c(x(y), y)) = \partial_y c(x(y), y).
  \end{equation}
  the second equality is due to the envelope theorem \cite{afriat1971theory}. Then $\nabla^2 f(y)$ can be computed as
  \begin{equation}
    \nabla^2 f(y) = \partial_{yx} c(x(y), y)\nabla x(y) + \partial_{yy} c(x(y), y).  \label{monge proof nabla2f}
  \end{equation}

  Plugging \eqref{implicit df} into \eqref{monge proof nabla2f}, recall \eqref{hessian Psi_x}, this yields
  \begin{equation*}
    \nabla^2 \Psi_x(y) = -\partial_{yx} c(x(y), y)\partial_{xx}\Phi(x(y), y)^{-1}\partial_{xy}c(x(y), y) + \partial_{yy}^2c(x(y), y) - \partial_{yy}^2c(x,y).
  \end{equation*}
  Recall the Assumption \ref{assmpt cost}, notice that $c\in C^2(\mathbb{R}^n\times\mathbb{R}^n)$, by symmetry of second derivatives, one can verify $\partial_{xy}c^\top = \partial_{yx}c$; Since $\partial_{yy}c(x,y)$ is independent of $x$, one has $\partial_{yy}^2c(x(y), y) - \partial_{yy}^2c(x,y)=0$. Thus we obtain
  \begin{equation}
    \nabla^2\Psi_x(y) = -\partial_{xy}^\top c(x(y), y)\partial_{xx}\Phi(x(y), y)^{-1}\partial_{xy}c(x(y), y).  \label{hessian of psi}
  \end{equation}
  By \eqref{boundness of Phi}, $\lambda(y) I_n \succeq \partial_{xx}\Phi(x(y), y)  \succ  O_n$.
  Recall that  $\sigma_{\min}(\partial_{xy}c(x,y))=\sigma(x,y)$. Now we apply lemma \ref{thm matrix pos def} to \eqref{hessian of psi}, this yields
  \begin{equation*}
    \nabla^2\Psi_x(y) \preceq  -\frac{\sigma(x,y)^2}{\lambda(y)}I_n.
  \end{equation*}

\end{proof}

Now we prove the main result of  Theorem \ref{thmerror est}:
\begin{proof}[Proof of Theorem \ref{thmerror est}]

  In this  proof, we   denote $\int$ as $\int_{\mathbb{R}^d}$ for simplicity.

  We first recall
  \begin{align*}
    \mathcal{L}(T,f) = & \int c(x, T(x))\rho_a(x)~dx + \int f(y)\rho_b(y)~dy - \int f(T(x))\rho_a(x)~dx \\
    =                  & \int f(y)\rho_b(y)~dy - \int (f(T(x)) - c(x,T(x)))\rho_a(x)dx,
  \end{align*}
  then we write
  \begin{align*}
    \mathcal{E}_1(T,f)  = \mathcal{L}(T,f) - \inf_{\widetilde{T}}\mathcal{L}(\widetilde{T}, f) = -\int [f(T(x)) - c(x,T(x))]\rho_a~dx + \sup_{\widetilde{T}} \left\{\int[f(\widetilde{T}(x)) - c(x,\widetilde{T}(x))]\rho_a~dx\right\}
  \end{align*}
  We denote
  \begin{equation}
    T_f(x) = \textrm{argmax}_{y}\{f(y)-c(x,y)\} = \textrm{argmax}_y\{\Psi_x(y)\},  \label{def Tf}
  \end{equation}
  recall that we denote $\Psi_x(y) = f(y)- c(x,y)$, then we have
  \begin{equation}
    \nabla\Psi_x(T_f(x)) = 0. \label{critical pt}
  \end{equation}
  One can also write:
  \begin{align*}
    \mathcal{E}_1(T,f) = & \int [(f(T_f(x))-c(x,T_f(x))) - ( f(T(x)) - c(x,T(x)) )] \\
    =                    & \int [\Psi_x(T_f(x))-\Psi_x(T(x))]\rho_a(x)~dx
  \end{align*}
  For any $x\in\mathbb{R}^d$, since $\Psi_x(\cdot)\in C^2(\mathbb{R}^n)$, and according to the previous Lemma \ref{concavity lemm}, we have
  \begin{equation*}
    \Psi_x(T(x)) - \Psi_x(T_f(x)) = \nabla\Psi_x(T_f(x))(T(x)-T_f(x)) + \frac{1}{2}(T(x)-T_f(x))^\top\nabla^2\Psi_x(\omega(x))(T(x)-T_f(x))
  \end{equation*}
  with $\omega(x)=(1-\theta_x)T(x)+\theta_x T_f(x)$ for certain $\theta_x\in[0,1]$. By \eqref{critical pt} and Lemma \ref{concavity lemm}, we have
  \begin{equation*}
    \Psi_x(T(x)) - \Psi_x(T_f(x)) \leq -\frac{\sigma(x,\omega(x))^2}{2\lambda(\omega(x))}|T(x) - T_f(x)|^2.
  \end{equation*}
  Thus we have:
  \begin{equation}
    \mathcal{E}_1(T,f) = \int [\Psi_x(T_f(x))-\Psi_x(T(x))]\rho_a(x)~dx \geq \int \frac{\sigma(x,\omega(x))^2}{2\lambda(\omega(x))}|T(x) - T_f(x)|^2\rho_a(x)~dx  \label{est 1}
  \end{equation}

  On the other hand, recall a saddle point solution of our max-min problem \eqref{max - min} is denoted as $(\bar{T}, \bar{f})$. Then we have %
  \begin{equation*}
    \sup_f\inf_T\mathcal{L}(T,f) = \int c(x,\bar{T}(x))\rho_a~dx - \int f(y)(\bar{T}_\sharp \rho_a - \rho_b)dy = \int c(x,\bar{T}(x))\rho_a~dx,
  \end{equation*}
  the second equality is due to the assumption on the consistency between $\bar{T}$ and Monge map $T_*$, thus $\bar{T}_\sharp \rho_a = \rho_b$.

  Thus we have
  \begin{align*}
    \mathcal{E}_2(f) = & \int c(x,\bar{T}(x))\rho_a~dx - \inf_{\widetilde{T}}\left(\int c(x, \widetilde{T}(x))\rho_a(x)dx +\int f(y)\rho_b(y)~dy - \int f(\widetilde{T}(x))\rho_a(x)dx\right) \\
    =                  & - \int f(\bar{T}(x)) - c(x,\bar{T}(x)) \rho_a~dx + \sup_{\widetilde{T}}\int (f(\widetilde{T}(x)) - c(x,\widetilde{T}(x)))\rho_a(x)dx.
  \end{align*}
  The second equality is due to $\bar{T}_\sharp\rho_a=\rho_b$. Similar to the previous treatment, we have
  \begin{equation*}
    \mathcal{E}_2(f) = \int [\Psi_x(T_f(x)) - \Psi_x(\bar{T}(x))]\rho_a(x)~dx
  \end{equation*}
  Apply similar analysis as before, we obtain
  \begin{equation}
    \mathcal{E}_2(f) \geq \int \frac{\sigma(x,\xi(x))^2}{2\lambda(\xi(x))}|\bar{T}(x) - T_f(x)|^2\rho_a(x)~dx   \label{est 2}
  \end{equation}
  with $\xi(x)=(1-\tau_x)\bar{T}(x)+ \tau_x T_f(x)  $ for certain $\tau_x\in[0,1]$. Since $\bar{T}=T_*$, $\rho_a$ almost surely, \eqref{est 2} leads to
  \begin{equation}
    \mathcal{E}_2(f) \geq \int \frac{\sigma(x,\xi(x))^2}{2\lambda(\xi(x))}|T_*(x) - T_f(x)|^2\rho_a(x)~dx   \label{est 2}
  \end{equation}

  Now we set
  \begin{equation}
    \beta(x) = \min\left\{\frac{\sigma(x,\omega(x))}{2\lambda(\omega(x))}, \frac{\sigma(x,\xi(x))}{2\lambda(\xi(x))}\right\} , \label{def beta}
  \end{equation}
  combining \eqref{est 1} and \eqref{est 2}, we obtain
  \begin{align*}
    \mathcal{E}_1(T,f) + \mathcal{E}_2(f) & \geq \int\beta(x)(|T(x)-T_f(x)|^2 + |T_*(x)-T_f(x)|^2)\rho_a~dx \\
                                          & \geq \int \frac{\beta(x)}{2}|T(x)-T_*(x)|^2\rho_a~dx
  \end{align*}
  This leads to $\|T-T_*\|_{L^2(\beta\rho_a)}\leq \sqrt{2(\mathcal{E}_1(T,f) + \mathcal{E}_2(f))}$.

\end{proof}

\section{Additional results}\label{add }

\subsection{Synthetic datasets}
\paragraph{Learning with unequal dimensions}
Our algorithm framework enjoys a distinguishing quality that it can learn the map from a lower dimension space $\mathbb{R}^{n}$ to a manifold in a higher dimension space $\mathbb{R}^{m} (n \leq m)$. In this scenario, we make the input dimension of neural network $T$ to be $n$ and output dimension to be $m$. In case the cost function $c(x,y)$  requires dimensions are $x$ and $y$ are equal dimensional, we patch zeros behind each sample $X \sim \rho_a$ and complement to a counterpart sample $\widetilde{X}=[X;\b0]$, where dimension of  $\b0$ is $m-n$.
And the targeted min-max problem is replaced by
\begin{align*}
  \max_\eta \min_\theta  \frac{1}{N} \sum_{k=1}^N c( {\widetilde{ X}_k}, T_\theta( {{ X}_k})) - f_\eta(T_\theta( {{ X}_k})) +f_\eta( {Y_k}).
\end{align*}
In Figure \ref{fig:unequal dim}, we conduct one experiment for $n=1$ and $m=2$.
The incomplete ellipse is a 1D manifold and our algorithm is able to learn a symmetric map from $\mathcal{N}(0,1)$ towards it.
\begin{figure}[h]
  \vspace{-0.4cm}
  \centering
  \begin{subfigure}{.2\textwidth}
    \centering
    \includegraphics[width=1\linewidth]{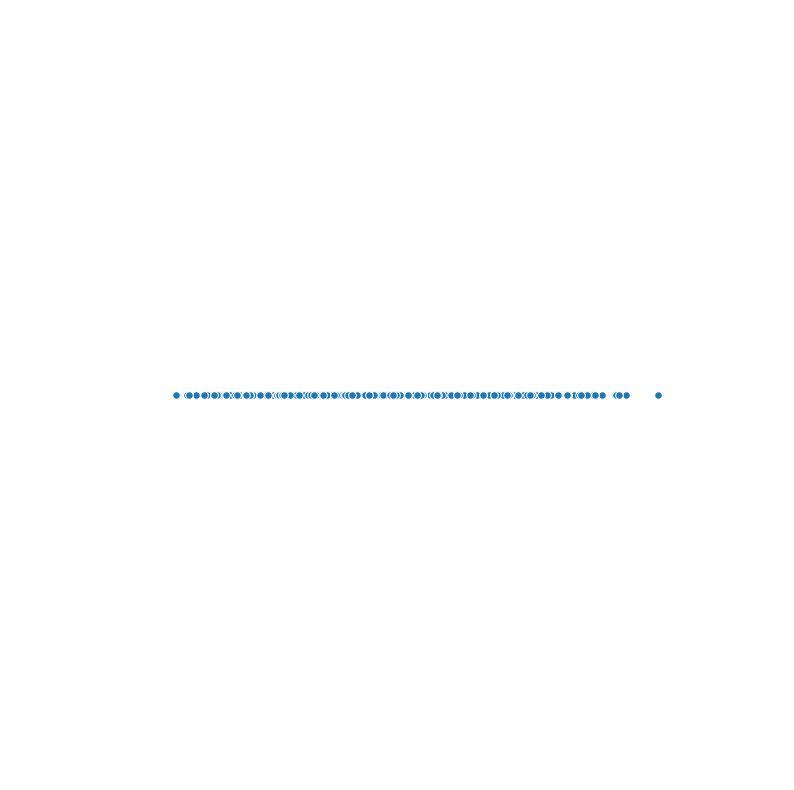}
    \caption{$\rho_a$}
  \end{subfigure}
  \begin{subfigure}{.2\textwidth}
    \centering
    \includegraphics[width=1\linewidth]{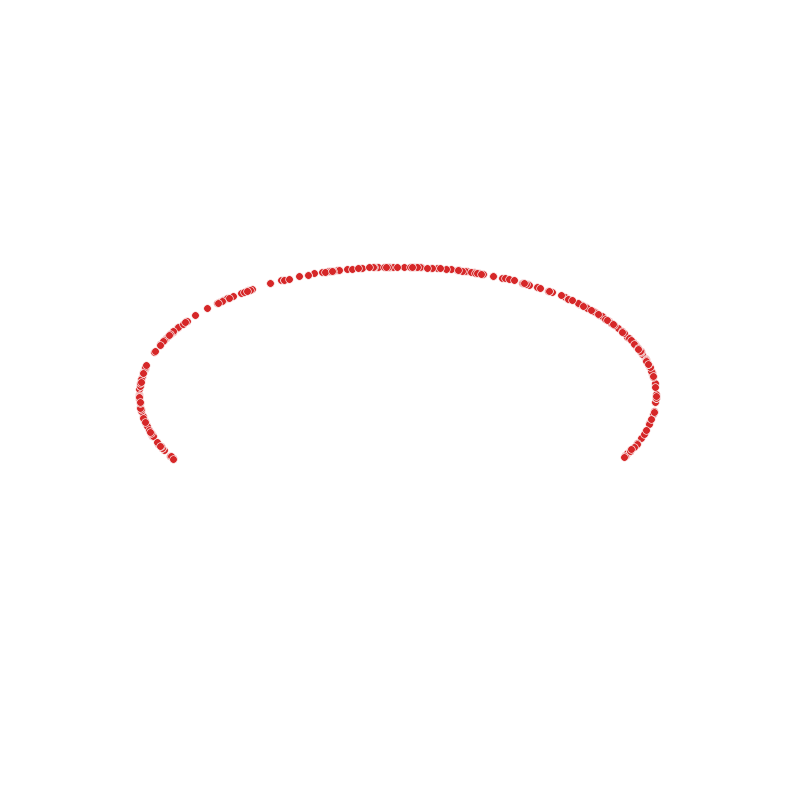}
    \caption{$\rho_b$}
  \end{subfigure}
  \begin{subfigure}{.2\textwidth}
    \centering
    \includegraphics[width=1\linewidth]{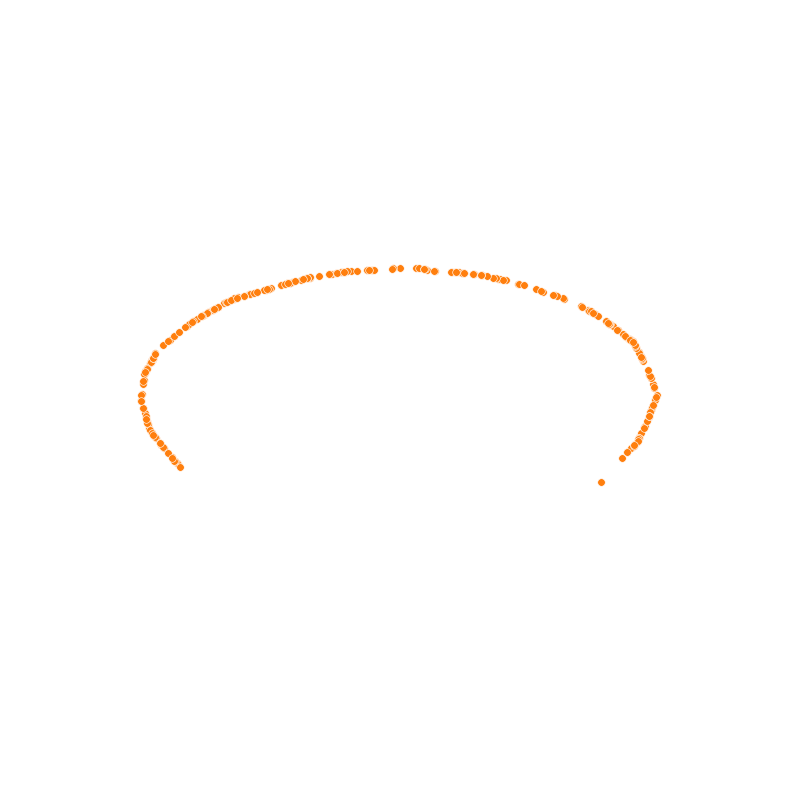}
    \caption{$T \sharp \rho_a$ }
  \end{subfigure}
  \begin{subfigure}{.2\textwidth}
    \centering
    \includegraphics[width=1\linewidth]{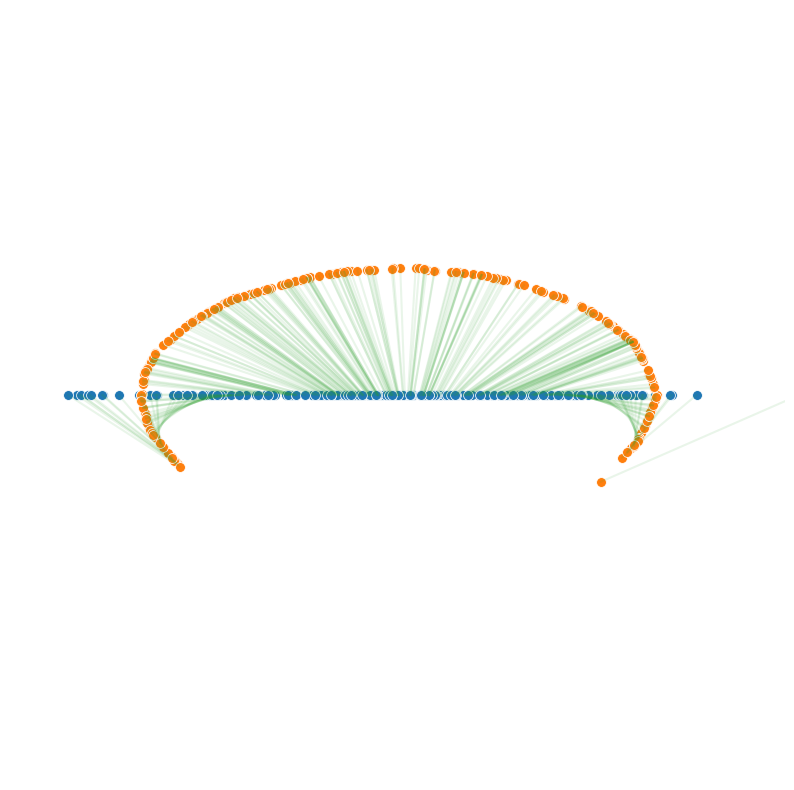}
    \caption{$T(\cdot)$ map }
  \end{subfigure}
  \caption{Qualitative results for learning unequal dimension maps. $\rho_a$ for two examples are both $\mathcal{N}(0,1)$, and $\rho_b$ are uniformly distributed on a incomplete ellipse and a ball respectively.}
  \label{fig:unequal dim}
  \vspace{-0.5cm}
\end{figure}

\paragraph{Decreasing function as the cost} We consider the cost function $c(x,y)=\phi(|x-y|)$ with $\phi$ as a monotonic decreasing function. We test our algorithm for a specific example $\phi(s)=\frac{1}{s^2}$. In this example, we compute for the optimal Monge map from $\rho_a$ to $\rho_b$ with $\rho_a$ as a uniform distribution on $\Omega_a$ and $\rho_b$ as a uniform distribution on $\Omega_b$, where we define
\begin{equation*}
  \Omega_a = \{(x_1,x_2)~|~ 6^2 \geq x_1^2+x_2^2 \geq 4^2 \},\quad \Omega_b = \{(x,x_2)~|~2^2\geq x_1^2+x_2^2 \geq 1^2 \}.
\end{equation*}
We also compute the same problem for $L^2$ cost. Figure \ref{fig:decrease cost} shows the transported samples as well as the differences between two cost functions.
\begin{figure}[h]
  \centering
  \begin{subfigure}{.186\textwidth}
    \centering
    \includegraphics[width=\linewidth]{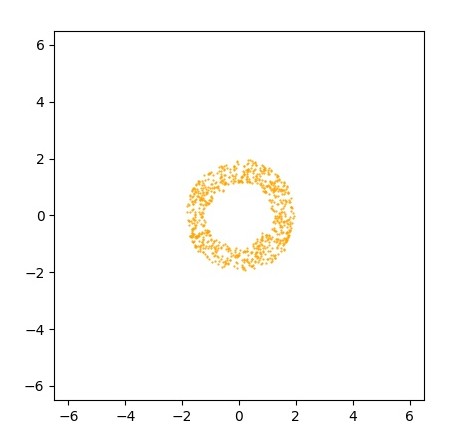}
    \caption{}
  \end{subfigure}
  \begin{subfigure}{.186\textwidth}
    \centering
    \includegraphics[width=1\linewidth]{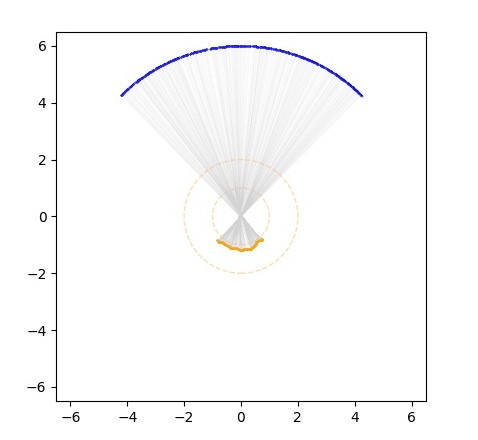}
    \caption{}
  \end{subfigure}
  \begin{subfigure}{.186\textwidth}
    \centering
    \includegraphics[width=1\linewidth]{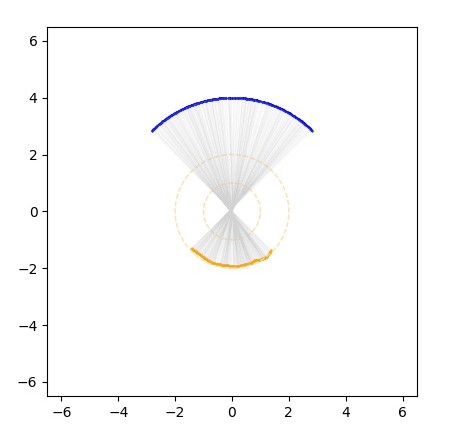}
    \caption{}
  \end{subfigure}
  \begin{subfigure}{.186\textwidth}
    \centering
    \includegraphics[width=1\linewidth]{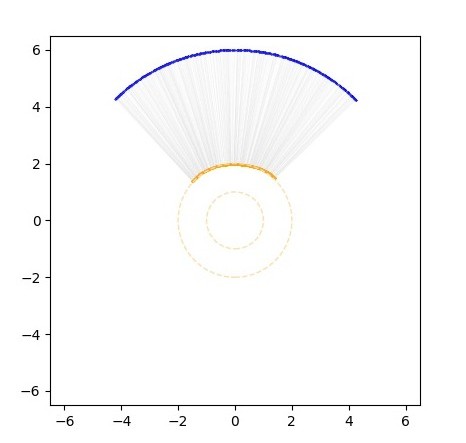}
    \caption{}
  \end{subfigure}
  \begin{subfigure}{.186\textwidth}
    \centering
    \includegraphics[width=1\linewidth]{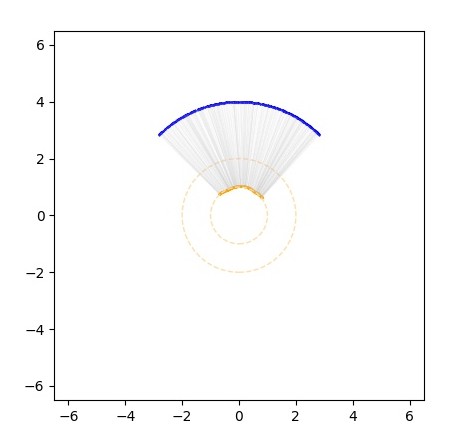}
    \caption{}
  \end{subfigure}
  \caption{(a) samples of computed $T_\sharp\rho_a$;  $c(x,y)=\frac{1}{|x-y|^2}$: Computed Monge map of quarter circles with radius $6$ (subplot b) and radius $4$ (subplot c);  $c(x,y)=|x-y|^2$: Computed Monge map of quarter circles with radius $6$ (subplot d) and radius $4$ (subplot e).}
  \label{fig:decrease cost}
\end{figure}

\paragraph{Uniform distribution on sphere
} For a given sphere $S$ with radius $R$, for any two points $x,y\in S$, we define the distance $d(x,y)$ as the length of the geodesic joining $x$ and $y$. Now for given $\rho_a$, $\rho_b$ defined on $S$, we consider solving the following Monge problem on $S$
\begin{equation}
  \min_{T,~T_\sharp\rho_a = \rho_b} \int_S d(x,T(x))\rho_a(x)~dx.  \label{spherical monge problem}
\end{equation}
Such sphere OT problem can be transferred to an OT problem defined on angular domain $D = [0, 2\pi)\times [0, \pi]$, to be more specific, we consider $(\theta, \phi)$ ($\theta\in[0,2\pi)$, $\phi\in [0,\pi]$) as the azimuthal and polar angle of the spherical coordinates. For two points $x=(R\sin\phi_1\cos\theta_1,R\sin\phi_1\sin\theta_1,R\cos\phi_1)$, $y=(R\sin\phi_2\cos\theta_2,R\sin\phi_2\sin\theta_2,R\cos\phi_2)$ on $S$, the geodesic distance
\begin{equation*}
  d(x,y) = c((\theta_1,\phi_1), (\theta_2,\phi_2)) = R\cdot\arccos(\sin\phi_1\sin\phi_2\cos(\theta_2-\theta_1) + \cos\phi_1\cos\phi_2).
\end{equation*}
Denote the corresponding distribution of $\rho_a,\rho_b$ on $D$ as $\hat{\rho}_a,\hat{\rho}_b$, now \eqref{spherical monge problem} can also be formulated as
\begin{equation}
  \min_{\hat{T},\hat{T}_{\sharp}\hat{\rho}_a = \hat{\rho}_b} \int c((\theta,\phi), \hat{T}(\theta,\phi))\hat{\rho}_a~d\theta d\phi.  \label{tran spherical monge probelm}
\end{equation}
We set $\hat{\rho}_a=U([0, 2\pi])\otimes U([0,\frac{\pi}{4}])$ and $\hat{\rho}_b = U([0, 2\pi])\otimes U([\frac{3\pi}{4},\pi])$. We apply our algorithm to solve \eqref{tran spherical monge probelm} and then translate our computed Monge map back to the sphere $S$ to obtain the following results

\begin{figure}[h]
  \centering
  \begin{subfigure}{.24\textwidth}
    \centering
    \includegraphics[width=1\linewidth]{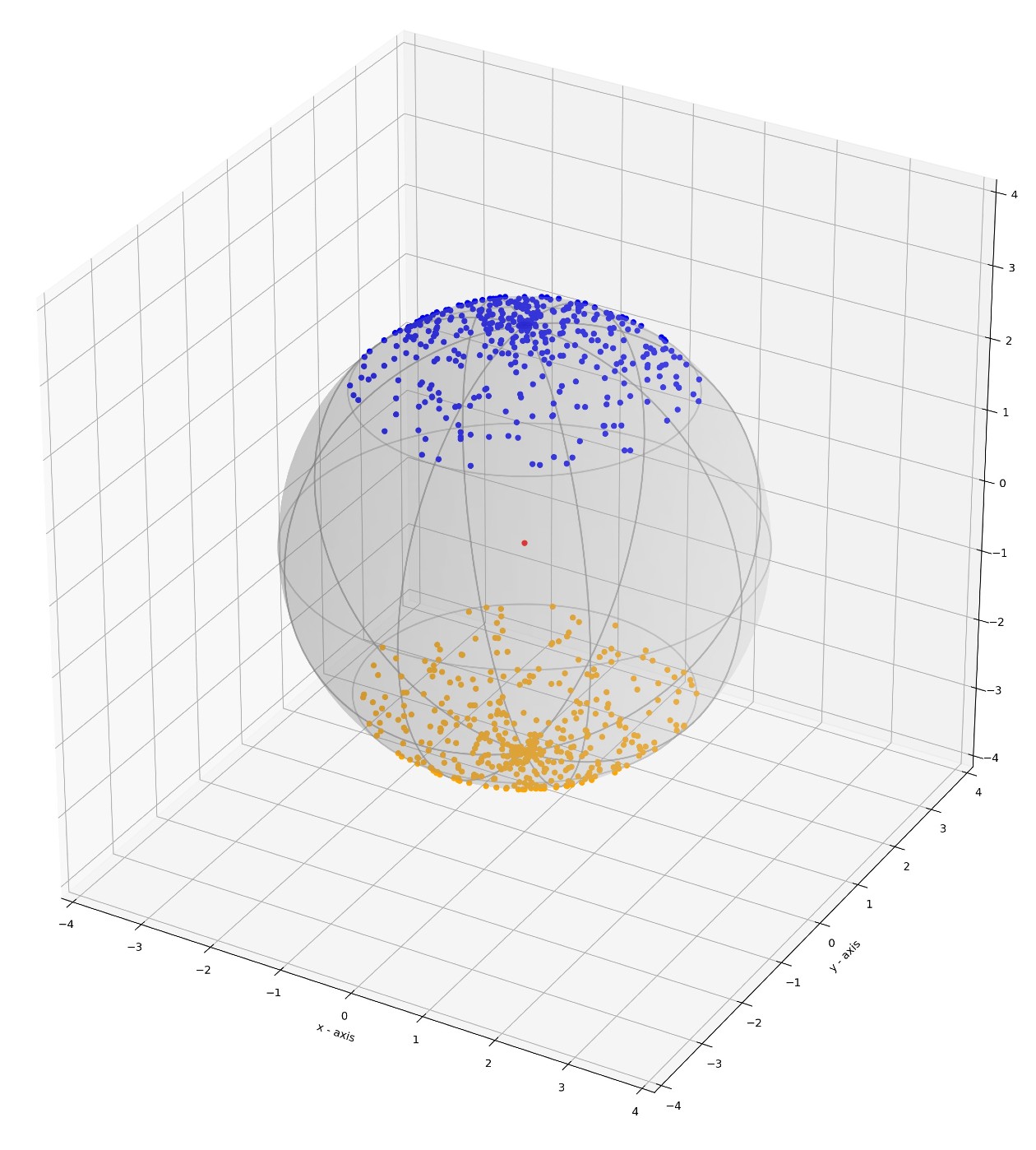}
    \caption{}
  \end{subfigure}
  \begin{subfigure}{.24\textwidth}
    \centering
    \includegraphics[width=1\linewidth]{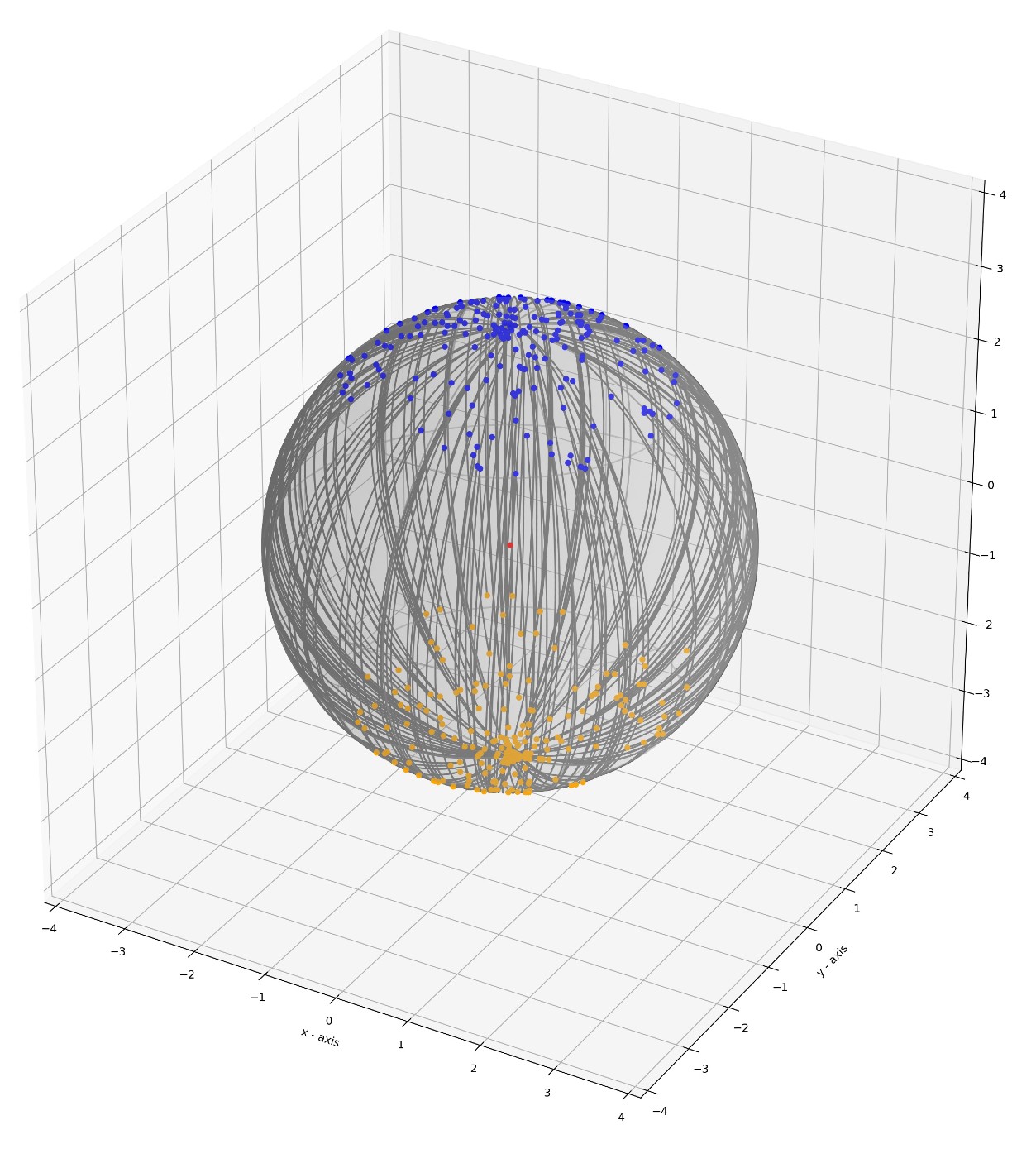}
    \caption{}
  \end{subfigure}
  \begin{subfigure}{.24\textwidth}
    \centering
    \includegraphics[width=1\linewidth]{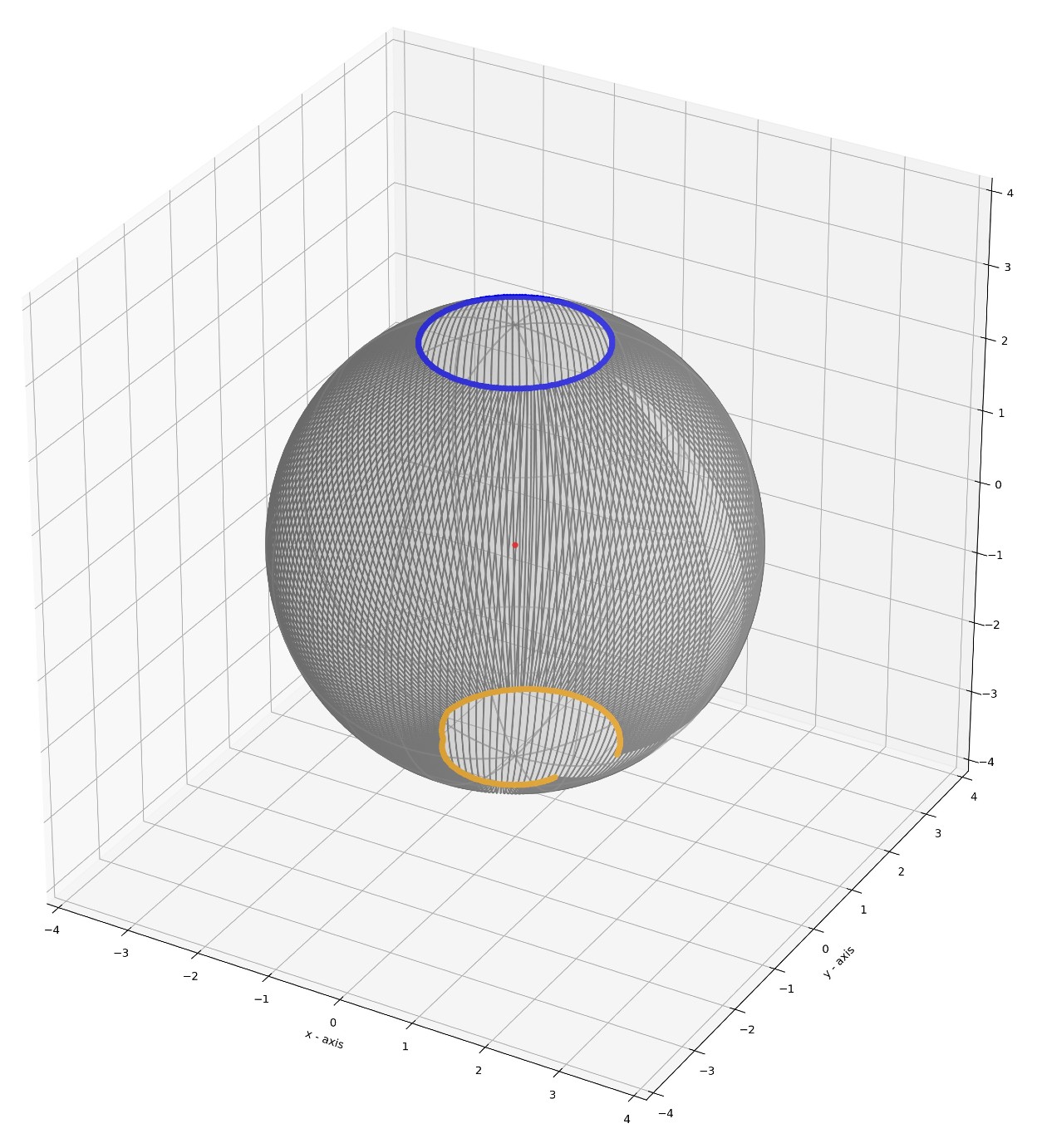}
    \caption{}
  \end{subfigure}
  \begin{subfigure}{.24\textwidth}
    \centering
    \includegraphics[width=1\linewidth]{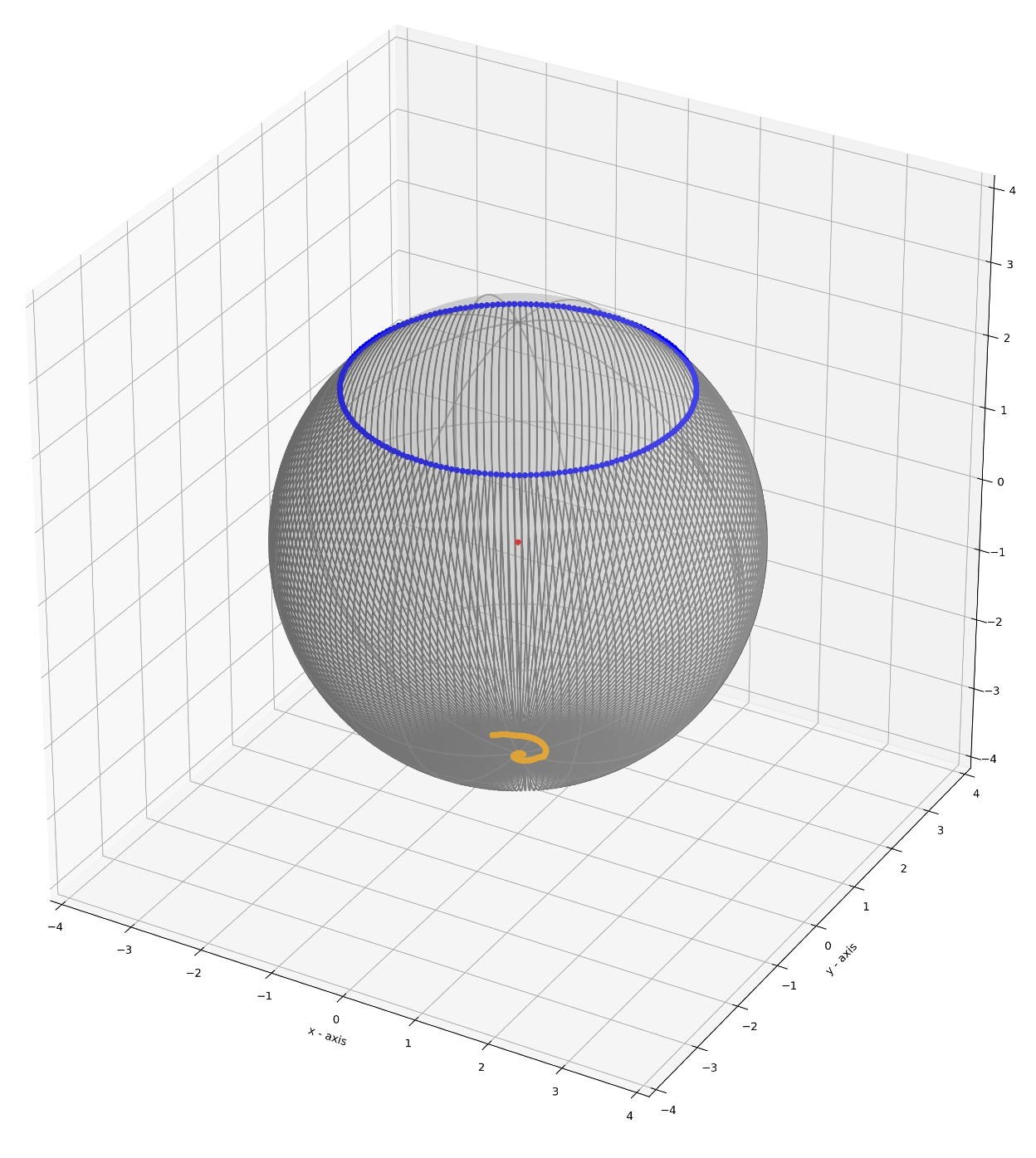}
    \caption{}
  \end{subfigure}
  \caption{Monge map from $\rho_a$ to $\rho_b$ on the sphere: (a) blue samples from $\rho_a$ (corresponds to $\hat{\rho}_a$) and orange samples from $\rho_b$ (corresponds to $\hat{\rho}_b$); (b) blue samples from $\rho_a$, orange samples are obtained from $\hat{T}_\sharp\hat{\rho}_b$, grey curves are geodesics connecting each transporting pairs; (c) our computed Monge map maps blue ring ($\phi=\frac{\pi}{8}$) to the orange curve (ground truth is $\phi=\frac{7}{8}  \pi $); (d) our computed Monge map maps blue ring ($\phi = \frac{\pi}{4}$) to the orange curve (ground truth is the southpole)}
  \vspace{-0.5cm}
\end{figure}

\vspace{-0.3cm}
\paragraph{Population transportation} We have described this example in Section \ref{population transportation}. Here we make further explanation on the map-to-land transform $\tau$, it is defined as follows.
\begin{equation*}
  \tau(\theta,\phi) = \begin{cases}
    (\theta,\phi), \quad \textrm{if}~ (\theta,\phi)\in\textrm{Land}; \\
    \underset{(\tilde{\theta},\tilde{\phi})\in P}{\textrm{argmin}}\left\{\|(\tilde{\theta},\tilde{\phi}) - (\theta, \phi)\|_2\right\}, \quad \textrm{if}~(\theta,\phi)\in\textrm{Sea}.
  \end{cases}
\end{equation*}
Here we choose $P$ as a finite set consists of $2000$ samples randomly selected from $\rho_b^{\textrm{Sph}}$.

We further plot in Figure \ref{fig: more sample plot} with sufficiently large amount of source samples and pushforwarded samples. Among the 40000 pushforwarded samples, 7718 are located on the sea, and we apply $\tau$ to map these samples back to a rather close location on land. It is worth mentioning that the map used as the background in our figures has several small regions removed from the actual land. This is due to the lack of data points in the dataset provided in \cite{geodata}.
\begin{figure}
  \centering
  \includegraphics[width=1\textwidth]{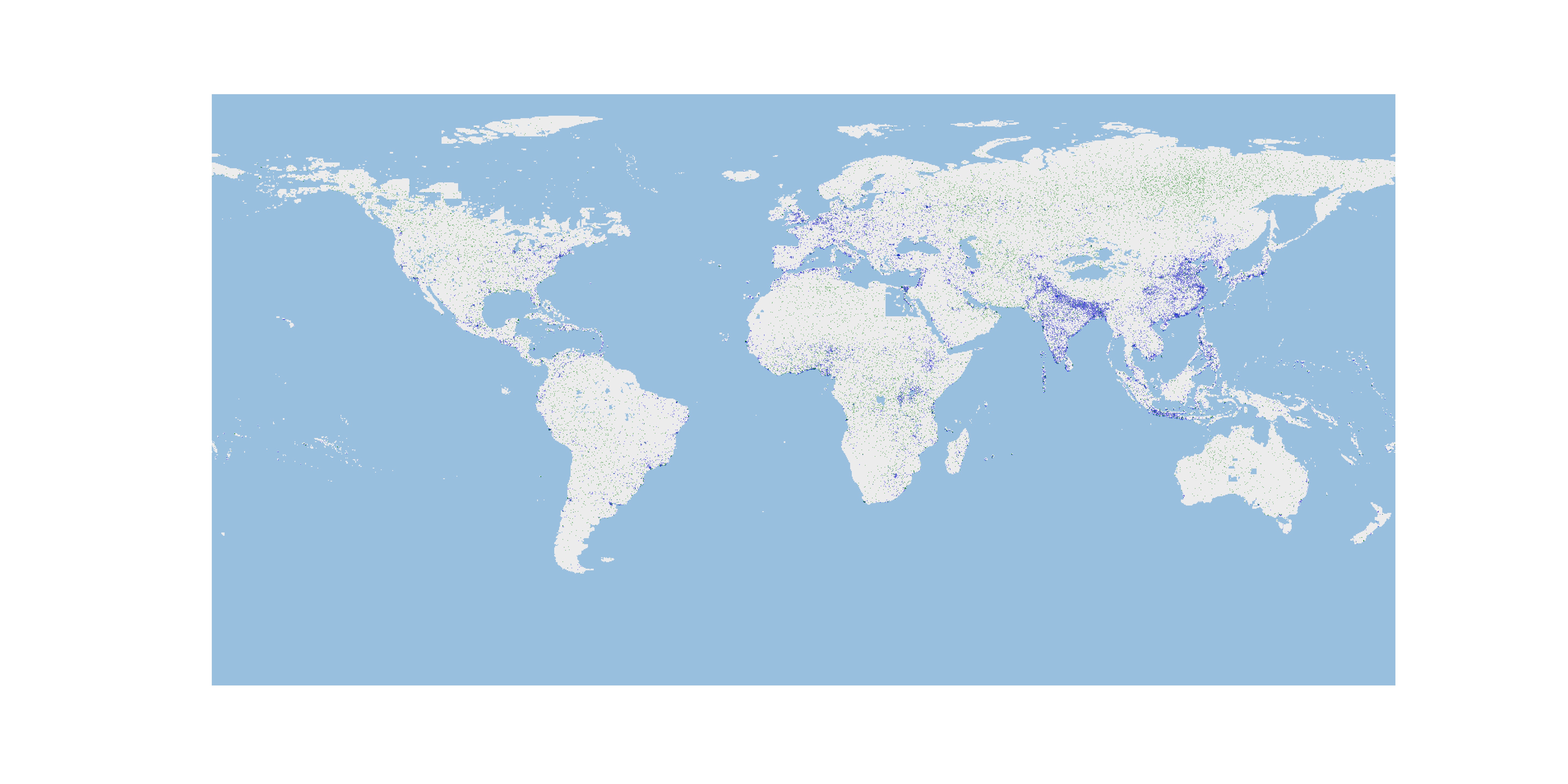}
  \vspace{-1.5cm}
  \caption{In this figure, we plot $N=40000$ samples $\{\theta_k,\phi_k\}_{k=1}^{N}$ (blue) randomly drawn from $\rho_a^{\textrm{Sph}}$, and their pushforwarded samples $\{T_{\theta}(\theta_k, \phi_k)\}_{k=1}^{N}$ (green). }
  \label{fig: more sample plot}
\end{figure}

\begin{figure*}[ht!]
  \centering
  \begin{subfigure}{1.0\textwidth}
    \centering
    \includegraphics[width=1\linewidth]{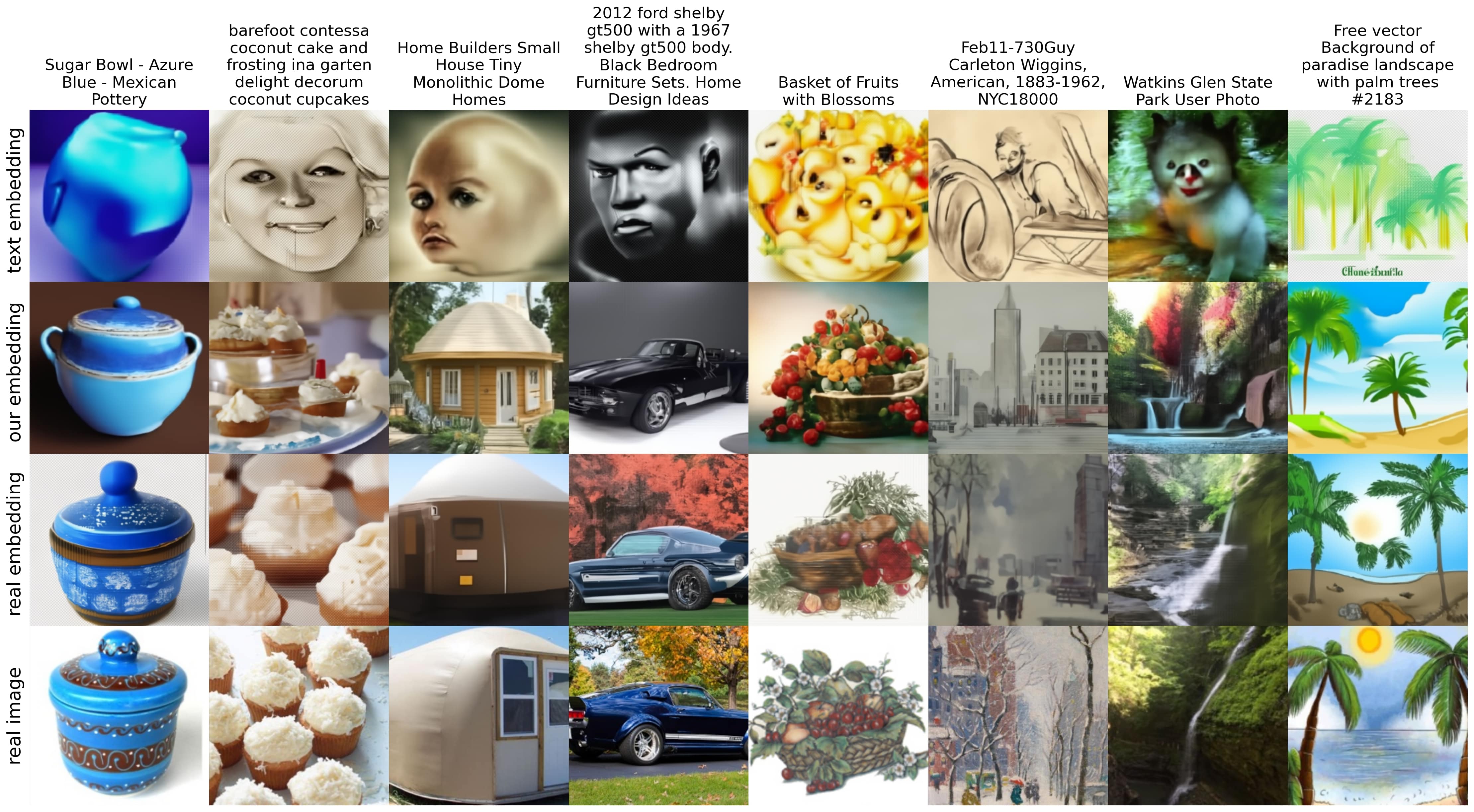}
    \caption{Random image samples on Laion art prompts}
  \end{subfigure}
  \begin{subfigure}{1.0\textwidth}
    \centering
    \includegraphics[width=1\linewidth]{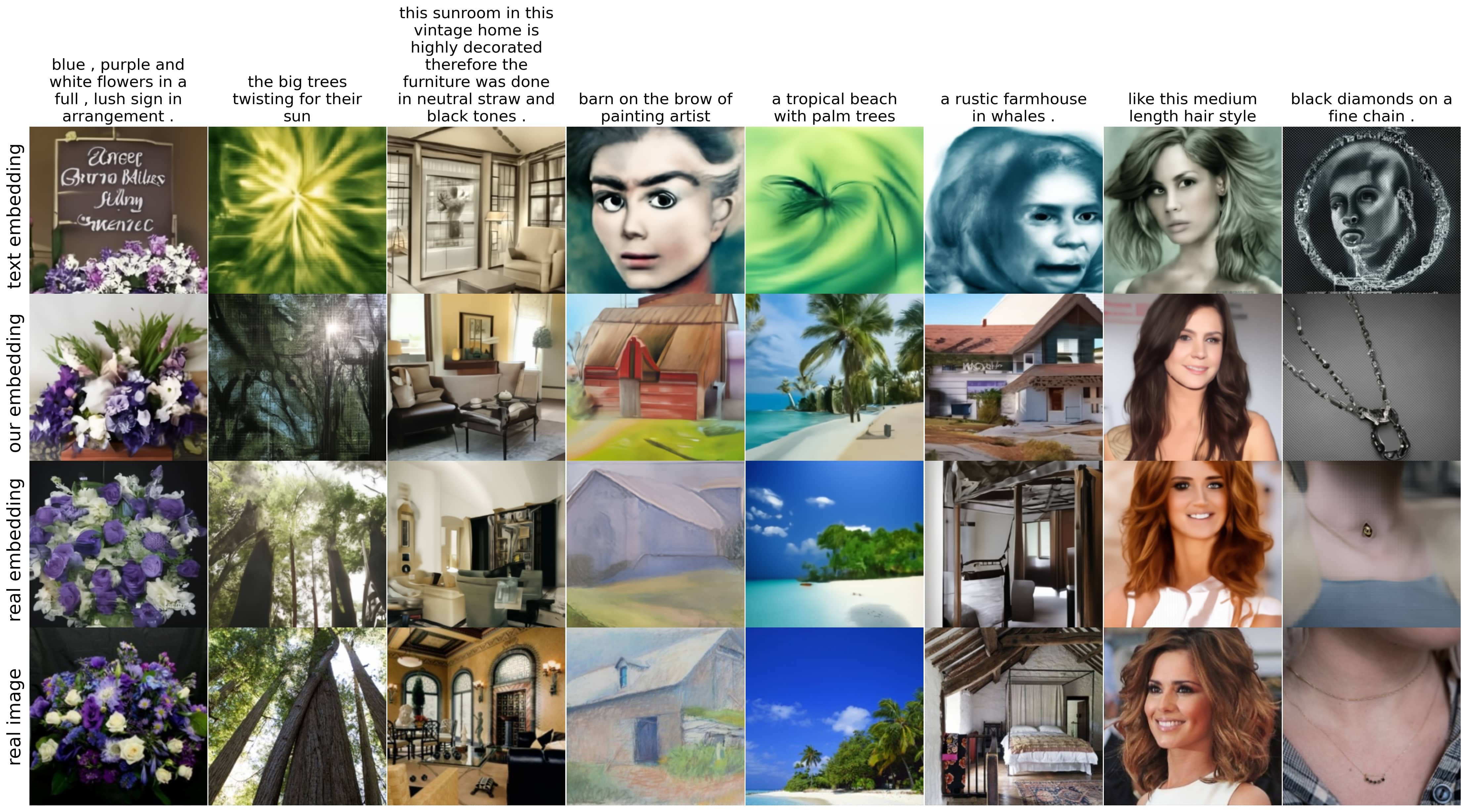}
    \caption{Random image samples on Conceptual Captions 3M prompts}
  \end{subfigure}
  \caption{Additional text to image generation results}
  \label{fig:additional_txt2img}
\end{figure*}

\begin{figure*}[ht!]
  \centering
  \begin{subfigure}{0.3\textwidth}
    \centering
    \includegraphics[width=1\linewidth]{images/degraded_images_d64.png}
    \includegraphics[width=1\linewidth]{images/real_img_d64.png}
    \caption{Degraded and original images}
  \end{subfigure}
  \begin{subfigure}{0.3\textwidth}
    \centering
    \includegraphics[width=1\linewidth]{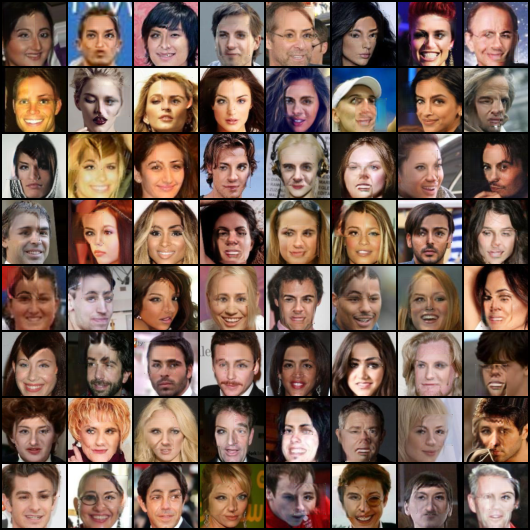}
    \includegraphics[width=1\linewidth]{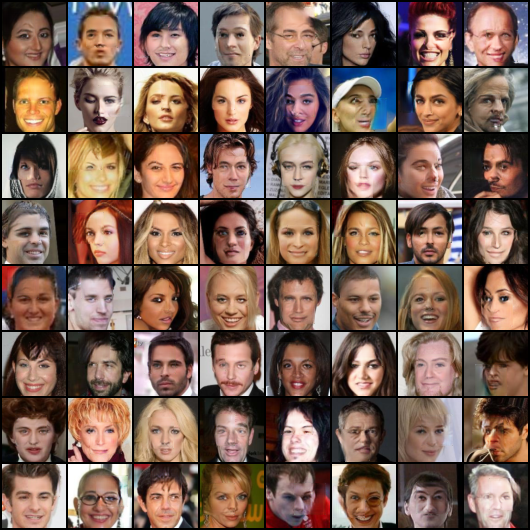}
    \caption{Composite images $G(x)$}
  \end{subfigure}
  \begin{subfigure}{0.3\textwidth}
    \centering
    \includegraphics[width=1\linewidth]{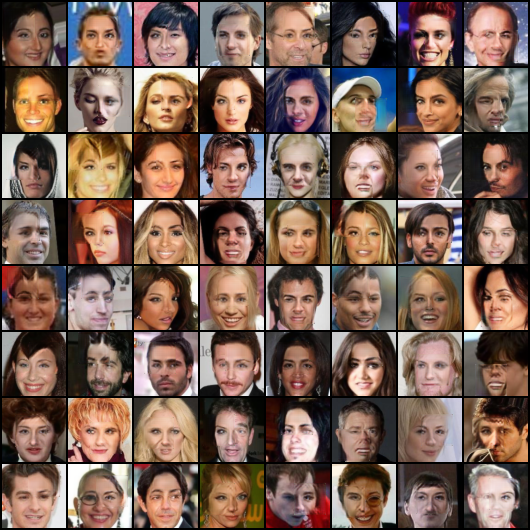}
    \includegraphics[width=1\linewidth]{images/pushed_images_d64_c1e4.png}
    \caption{Pushforward images $T(x)$ }
  \end{subfigure}
  \caption{Unpaired image inpainting on \textbf{test} dataset of CelebA $64\times 64$.
    We take the composite image
    $G(x)= T(x) \odot M^C + x \odot M$
    as the output image. Additionally, we provide the pushforward images $T(x)$ to illustrate the regularization effect of transportation cost.
    In panel (b) and (c), we show the results with $\alpha=10$ in the first row and $\alpha =10000 $ in the second row.
    A small transportation cost would result that pushforward map neglects the connection to the unmasked area, which is illustrated by a more clear mask border in pushforward images.
  }
  \label{fig:celeba64_addition}
\end{figure*}

\subsection{Real-world dataset}

We show additional text to image generation results by our algorithm in Figure \ref{fig:additional_txt2img}, and additional CelebA 64$\times$64 inpainting results in Figure \ref{fig:celeba64_addition}.
Moreoever, we evaluate our algorithm in the following scenario.

\paragraph{Class-preserving map}
We consider the class-preserving map between two labelled datasets.
The input distribution $\rho_a = \sum_{j=1}^J a_j \rho^j_a $ is a mixture of  $J$ distinct distributions $\{\rho^j_a \}_{j=1}^J$.
Similarly, the target distribution is $\rho_b = \sum_{j=1}^J a_j \rho^j_b $. Each distribution $\rho_a^j$ (or $\rho_b^j$) is associated with a known label/class $j$.
We further assume that the support of $\{\rho_a^j\}$ (or $\rho_b^j$) are disjoint.
We seek a map that solves the  problem
\begin{equation}\label{eq:cond_monge}
  \min_{\substack{T_\sharp \rho^j_a = \rho^j_b}} \int_{\mathbb{R}^n}
  \|x -T(x) \|^2
  \rho_a(x)~dx,
\end{equation}
where the constraint asks the map to preserve the original class. To approximate this map, we replace the constraint by a contrastive penalty. Indeed, this trick transfers \eqref{eq:cond_monge} to the original \hyperref[Monge problem]{Monge problem} with a cost
\begin{align*}
  c(\{x,y\}, \{x',y'\} ) = \|x - x'\|^2 +  \mathbf{1}( y \neq y' ),
\end{align*}
where $y$ and $y'$ are the labels corresponding to $x$ and $x'$,
$\mathbf{1}$ is the indicator function.
To better guide the mapping, we involve the label as an input to the mapping $T$ and the potential $f$. Denote
$\ell(\cdot): \R^m \rightarrow \R^J $ as a pretrained classifier on the target domain $\rho_b$. Given a feature from the target domain, $\ell(\cdot)$ will output a probability vector.  Then,
our full formula reads
\begin{align*}
  \sup_f \inf_T & \int \left[ c(x, \bar x) -  y^\top
    \bar y - f(\bar x; \bar y) )\right] d\rho_a + \int f(x'; y')d \rho_b,               \\
                & \text{where } \quad  \bar x = T(x; y) , \quad \bar y = \ell(T(x; y)).
\end{align*}
In practice, $y \in \R^J$ is a one-hot label and it is known because both the source and target datasets are labelled.
Our method is different from \citet{asadulaev2022neural} because they only solve a map such that $T_\sharp \rho^j_a = \rho^j_b$, and ignore the transport cost. And their map does not involve the label as input.
The closest work to us is the covariate-guided conditional Monge map in \citet{bunne2022supervised}. They use PICNN~\citep{icnn} to approximate the conditional Monge map and the potential.

We compare our algorithm with \citet{asadulaev2022neural} on $^*$NIST~\citep{LeCun2005TheMD,xiao2017fashion,clanuwat2018deep} datasets. The visualization of our mapping is presented in Figure \ref{fig:nist}. We also calculate FID between the mapped source test dataset and the target test dataset. To quantify how well the map preserves the original class, we use a pretrained SpinalNet classifier~\citep{kabir2022spinalnet} on the target domain and evaluate the accuracy of the predicted label.
If the predicted label of $T(x;y)$ equals to the original label $y$, then the mapping is correct, otherwise wrong. We show the quantitative result in Table \ref{tab:class}, where the results of \citet{asadulaev2022neural} are from their paper. Our accuracy result is nearly $100\%$ thanks to the contrastive loss guidance. Even though our FID is slightly inferior, we conjecture this can be improved by using a stochastic map~\citep{korotin2022neural} as well.

\begin{figure*}[ht!]
  \centering
  \begin{subfigure}{0.4\textwidth}
    \centering
    \includegraphics[width=1\linewidth]{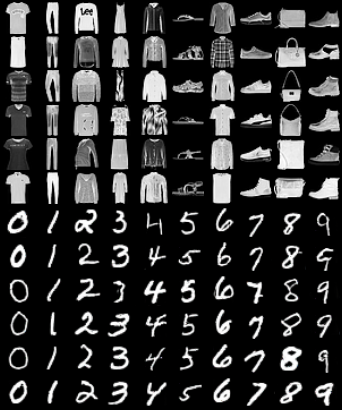}
    \caption{FMNIST $\rightarrow$ MNIST}
  \end{subfigure}
  \hspace{0.3cm}
  \begin{subfigure}{0.4\textwidth}
    \centering
    \includegraphics[width=1\linewidth]{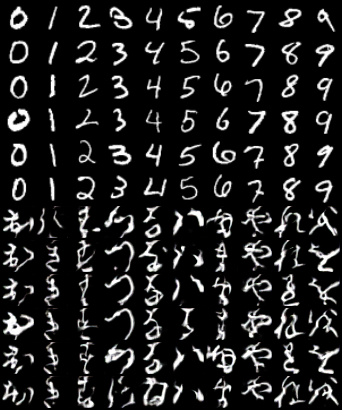}
    \caption{MNIST $\rightarrow$ KMNIST}
  \end{subfigure}
  \caption{Class-preserving mapping. The top block is the source images, and the bottom block is the pushforward images. Each column represents a class.}
  \label{fig:nist}
\end{figure*}

\begin{table}[]
  \caption{Accuracy of the maps and FID of generated samples.}
  \label{tab:class}
  \begin{center}
    \begin{tabular}{|c|c|c|c|c|}
      \hline
                                &                            & \begin{tabular}[c]{@{}c@{}}\citet{asadulaev2022neural}\\ deterministic map \end{tabular} & \begin{tabular}[c]{@{}c@{}}\citet{asadulaev2022neural} \\ stochastic map \end{tabular} & Ours       \\ \hline
      \multirow{2}{*}{FID}      & MNIST $\rightarrow$ KMNSIT & 17.26                       & \bf{9.69}                   & 18.07      \\ \cline{2-5}
                                & FMNIST $\rightarrow$ MNIST & 7.15                        & \bf{5.26}                   & 5.78       \\ \hline
      \multirow{2}{*}{Accuracy} & MNIST $\rightarrow$ KMNSIT & 79.20                       & 61.91                       & \bf{99.59} \\ \cline{2-5}
                                & FMNIST $\rightarrow$ MNIST & 82.97                       & 83.22                       & \bf{99.73} \\ \hline
    \end{tabular}
  \end{center}
\end{table}

\section{Implementation details and hyper-parameters}\label{hyperp}

\subsection{Synthetic datasets}

\paragraph{Unequal dimensions}
The networks $T_\theta$ and $f_\eta$ each has 5 layers with 10 hidden neurons. The batch size $B=100$. $K_1=6,K_2=1$. The learning rate is $10^{-3}$. The number of iterations $K=12000.$

\paragraph{Decreasing cost function}
In this example, we set $T_\theta(x) = x+F_\theta(x)$ and optimize over $\theta$. For either $\frac{1}{|x-y|^2}$ or $|x-y|^2$ case we set both  $F_\theta$ and the Lagrange multiplier $f_\eta$ as six layers fully connected neural networks, with PReLU and Tanh activation functions respectively, each layer has 36 nodes. The training batch size $B = 2000 $. We set $K=2000$, $K_1 = 8, K_2=6$.

\paragraph{On sphere}

In this example, we set $T_\theta(x) = x+F_\theta(x)$ and optimize over $\theta$. We set both  $F_\theta$ and $f_\eta$ as six layers MLP, with PReLU activation functions, each layer has 8 nodes The training batch size is $ 200 $. We set $K=4000$, $K_1 = 8, K_2= 4 $. We choose rather small learning rate in this example to avoid gradient blow up, we set $0.5\times 10^{-5}$ as the learning rate for $\theta$ and $10^{-5}$ as the learning rate for $\eta$.

\paragraph{Population transportation} In this example, we choose both $T_\theta, f_\eta$ as ResNets with depth equals $4$ and hidden dimensions equals $32$; we choose the activation function as PReLU. We apply dropout technique \cite{hintondropout} with $p=0.24$ to each layer of our networks. We follow the algorithm presented in Algorithm \ref{alg:1} to train the map $T_\theta$, We optimize with Adam method with learning rate $5\times 10^{-5}$. $T_\theta$ is computed after $200000$ steps of optimization. For the comparison experiment, we tested our example by slightly modifying the codes presented in  \href{https://pythonot.github.io/auto_examples/domain-adaptation/plot_otda_mapping.html}{OT mapping estimation for domain adaptation} of POT library \cite{flamary2021pot}. We train the linear transformation on $2000$ samples from both $\rho_a^{\textrm{Sph}}$ and $\rho_b^{\textrm{Sph}}$, we choose the hyper parameter $\mu=1, \epsilon = 10^{-4}$. We plot the second Figure \ref{spherical transport} by applying the map-to-land transform $\tau$ with $2000$ newly selected samples. We also tried the same example with Gaussian kernel, but the results are always unstable for various hyper parameters. Generally speaking, it is very hard to obtain valid results with Gaussian kernel.

\subsection{Text to image generation}

\paragraph{Dataset details}

\href{https://github.com/LAION-AI/laion-datasets/blob/main/laion-aesthetic.md}{Laion aesthetic} dataset is filtered from a Laion 5B dataset to have the high aesthetic level.
Laion art is the subset of Laion aesthetic and contains 8M most aethetic samples. We download the metadata of Laion art according to the
\href{https://github.com/rom1504/img2dataset/blob/main/dataset_examples/laion-art.md}{instructions}
of Laion. We then filter only English prompts, and download the images with \href{https://github.com/rom1504/img2dataset/blob/main/dataset_examples/laion5B.md}{img2dataset}. To speed up the training, we run the \href{https://github.com/rom1504/clip-retrieval}{CLIP retrieval} to convert images to embeddings and use \href{https://github.com/Veldrovive/embedding-dataset-reordering}{
  embedding-dataset-reordering} to reorder the embeddings into the expected format. By doing these two steps, we save the time of calculating embeddings on the fly. After filtering English prompts, we get the Laion art dataset with 2.2M data.

We download CC-3M following this \href{https://github.com/rom1504/img2dataset/blob/main/dataset_examples/cc3m.md}{instruction}. We then remove all the images with watermark, which include images downloaded from \textit{shutterstock, alamy, gettyimages,} and \textit{dailymail.co.uk} websites. After removing watermark images, the CC-3M only has 0.8M (text, image) pairs left.

For each dataset, we let the source and the target distribution contain
0.3M data respectively, and take the rest of dataset as the test data.

\paragraph{Network structure and hyper-parameters} We use the \href{https://github.com/lucidrains/DALLE2-pytorch}{DALLE$\cdot$2 prior network} to represent the map $T$. The potential $f$ is using the same network with an additional pooling layer. We use \href{https://github.com/fadel/pytorch_ema}{EMA} to stablize the training of the map.  The batch size is 225. The number of loop iterations are $K_1=10$, $K_2=1$. We use the learrning rates $10^{-4}$, Adam \citep{kingma2014adam} optimizer with weight decay coefficient $0.0602$. We train the networks for 110 epochs.

On NVIDIA RTX A6000 (48GB), the training time of  each experiment is 21 hours.

\subsection{Unpaired inpainting}

Recall the composite image is
$G(x)= T(x) \odot M^C + x \odot M$, where $M $ is the inpainting mask.
The loss function is slightly different with the \eqref{eq:L def}. We modify the $f(T(x))$ to be $f(G(x))$ to strengthen the training of $f$
\begin{align*}
  \sup_f \inf_T \int_{\mathbb{R}^n} \left[ c(x,T(x)) - f(G(x))\right]\rho_a(x)~dx + \int_{\mathbb{R}^m} f(y)\rho_b(y)~dy.
\end{align*}

In the unpaired inpainting experiments,  the images are first cropped at the center with size 140 and then resized to $64 \times 64 $ or $128 \times 128 $. We choose learning rate to be $1 \cdot 10^{-3}$, Adam optimizer with default beta parameters,
$K_2=1$. The batch size is 64 for CelebA64 and 16 for CelebA128. The number of inner loop iteration $K_1=5$ for CelebA64 and $K_1=10$ for CelebA128.

We use exactly the same UNet for the map $T$ and convolutional neural network for $f$ as \citet[Table 9]{rout2022generative} for CelebA64 and add one additional convolutional block in $f$ network for CelebA128.
On NVIDIA RTX A6000 (48GB), the training time of  CelebA64 experiment is 10 hours and the time of CelebA128 is 45 hours.

We use the \href{https://pythonot.github.io/auto_examples/domain-adaptation/plot_otda_mapping.html}{POT implementation} of \citet{perrot2016mapping} as the template and implement the masked MSE loss by ourself. \citet{perrot2016mapping} provides two options as the transformation map, one is linear and another is the kernel function. We find the kernel function on this example would have generate mode collapse results, so we adopt the linear transformation map. We choose $L_2 $ regularization coefficient to be $10^{-3}$, and other parameters are the same as default. \textbf{The result of discrete OT in Figure \ref{fig:celeba64} in the main paper was blue due to an incorrect channel order. We have corrected it in the supplementary material.}

\subsection{Class preserving mapping}

The NIST images are rescaled to size $32\times 32$ and repeated to 3 channels.

We use the \href{https://github.com/kgkgzrtk/cUNet-Pytorch}{conditional UNet} to represent map $T$. We add a projection module~\citep{miyato2018cgans} on WGAN-QC's ResNet~\citep{liu2019wasserstein} to represent potential $f$.
We choose learning rate to be $1 \cdot 10^{-4}$, Adam optimizer with betas $(0.5,0.999)$.
$K_1=10, K_2=1$.
We use \href{https://github.com/fadel/pytorch_ema}{EMA} to stablize the training of the map.
The batch size is 64. In practice, we introduce a coefficient $\lambda$ in the cost $ c(\{x,y\}, \{x',y'\} ) = \|x - x'\|^2 +  \lambda \mathbf{1}( y \neq y' )$
and set $\lambda =0.5$.

\end{document}